\documentclass[11pt, english]{article}

\usepackage{amsfonts}
\usepackage{textcomp}
\usepackage{xcolor}
\usepackage{booktabs}

\usepackage{dsfont}
\usepackage{amsmath}
\usepackage{graphicx}
\usepackage{amssymb}
\usepackage{amsthm}
\usepackage{bbm}
\usepackage{mathtools}
\usepackage{algorithmic}
\usepackage{courier} 
\usepackage{algorithm}
\graphicspath{{./figures/}}
\usepackage[round]{natbib}
\renewcommand{\bibname}{References}
\renewcommand{\bibsection}{\subsubsection*{\bibname}}
\usepackage{wrapfig}

\usepackage{chngcntr}
\usepackage{apptools}
\AtAppendix{\counterwithin{thm}{section}}

\usepackage{hyperref}
\hypersetup{hidelinks}
\DeclarePairedDelimiter{\ceil}{\lceil}{\rceil}

\newcommand{\norm}[1]{\left\Vert #1 \right\Vert}
\numberwithin{equation}{section}
\newcommand{\ee}{{\rm e}\hspace{1pt}}

\newcommand{\abs}[1]{\left| #1 \right|}
\newcommand{\veps}{\varepsilon}

\newcommand{\wh}{\widehat}
\newtheorem{thm}{Theorem}
\newtheorem{lem}[thm]{Lemma}
\newtheorem{cor}[thm]{Corollary}

\newtheorem{defn}[thm]{Definition}

\newtheorem{remark}[thm]{\textit{Remark}}

\newcommand{\wt}[1]{\widetilde{#1}}

\newcommand{\para}[1]{\textbf{#1}}

\newcommand{\tr}{\mathrm{tr}}

\title{Federated Differentially Private Optimization via \\  Low-Rank Preconditioning}

\author{
Antti Koskela, Tejas Kulkarni, Elias Raninen\\
Nokia Bell Labs
}

\date{}

\begin{document}

\maketitle

\begin{abstract}
We consider differentially private (DP) optimization in the federated setting, where a DP estimate of the feature covariance matrix serves as a Hessian approximation for generalized linear models. Existing central DP methods for covariance estimation have communication cost that grows quadratically in the data dimension, while existing baselines with linear communication cost are only locally private and degrade as the number of clients increases. As a remedy, we propose a federated DP-PCA framework that achieves linear communication cost while providing central DP guarantees under secure aggregation. It consists of a DP rank selection step, a DP power method, and a noisy release of the covariance projected onto the selected low-dimensional subspace. We give error bounds for this covariance estimate and for the resulting low-communication preconditioned DP gradient descent. Empirically, our proposed DP optimization method matches a centralized baseline with quadratic communication cost at a fraction of the communication.
\end{abstract}

\section{Introduction}

Estimating the covariance matrix of high-dimensional data is a fundamental primitive in machine learning. A low-rank approximation of the covariance often provides a useful proxy of the full covariance, can be interpreted through principal component analysis (PCA), and serves as a building block for tasks such as dimensionality reduction and denoising. In many realistic scenarios the data is distributed across several data-owners that cannot or do not want to centralize their data. This motivates federated algorithms that can estimate a global low-rank covariance while providing differential privacy (DP)~\citep{dworkOurDataOurselves2006} guarantees for the data owners. At the same time, for practical purposes it would often be crucial to have communication costs that scale linearly rather than quadratically with the feature dimension $p$.

Covariance estimates can also be used as Hessian approximations in optimization and can accelerate first-order DP optimization methods. Recent private preconditioned methods use DP covariance estimates to approximate second-order information~\citep{mehta2023differentially,krouka2025communication}, reaching the same privacy-utility trade-off as DP-SGD in fewer iterations~\citep[see also][]{ganesh2023faster}. In federated settings, where each iteration corresponds to a communication round, this directly reduces communication costs. However, obtaining the DP covariance estimate remains the computational and communication bottleneck. For example, DP-FC~\citep{mehta2023differentially} relies on so called Analyze Gauss mechanism~\citep{dwork2014analyze} applied to the full covariance matrix, which in a federated implementation requires $O(p^2)$ communication. To the best of our knowledge, there is no prior federated method for private second-order optimization that uses a communication-efficient DP low-rank covariance (Hessian proxy) estimate and simultaneously supports secure aggregation end-to-end. We fill this gap.

\para{DP-FL via secure summation.} While federated learning (FL)~\citep{FLbook} improves on naive collaborative learning by avoiding direct data transmission, the local updates can still leak information about training data, e.g.\ through membership inference attacks~\citep{shokri2017membership}, which DP protects against. Most DP-FL methods follow the same pattern: clients add Gaussian noise to their local messages (e.g., ML model updates) and the server aggregates them~\citep{truex2019hybrid}. When the noise is additive and the updates are sums of user contributions, secure aggregation protocols let the clients carry the noise injection themselves, removing the need to trust a central server~\citep{UllahCKO23,FLDP3}.\footnote{This assumes an honest-but-curious server and honest clients that each add their share of the noise; dropped or colluding clients reduce the effective noise. Since secure aggregation operates over finite groups, the Gaussian noise must in practice be discretized~\citep{FLDP3}; we analyze the continuous Gaussian for simplicity.}

\para{Related Work on DP covariance estimation.} The classical Analyze Gauss mechanism~\citep{dwork2014analyze} adds Gaussian noise directly to the $p \times p$ covariance matrix and post-processes it to obtain a low-rank approximation. The algorithm distributes naturally, but requires $O(p^2)$ communication, which becomes prohibitive in high feature dimensions $p$. Exponential-mechanism-based methods~\citep{chaudhuri2012near} and recent centralized refinements~\citep{kim2025robust,he2025differentially,cai2024optimal}, in contrast, do not decompose into local updates compatible with secure aggregation.

\para{Related Work on DP-FL covariance estimation.} There are existing works that consider federated algorithms with $O(p)$ communication~\citep{grammenos2020federated,guo2021fedpower,li2024federated}. All of these methods are locally DP: each client calibrates its noise to its local sensitivity and then applies post-processing (e.g., a truncated SVD), so that post-processed quantities are aggregated rather than sums. This also prevents obtaining central DP guarantees. As a consequence, their privacy--utility trade-off degrades as the number of clients grows. The federated DP power method of~\citet{nicolas2024differentially}~\citep[also studied by][]{balcan2016improved,campbell2025decentralized} does enjoy the benefits of summation, but focuses exclusively on subspace computation. Our contribution is to use this basis step in a complete end-to-end framework that also handles private covariance estimation and rank selection, to empirically show it leads to improvements in DP-FL optimization and to analyze its use as a preconditioner in DP-FL optimization.

\subsection{Our Contributions}

\begin{itemize}

    \item \textbf{Communication-efficient federated DP-PCA.} We introduce the first federated DP low-rank covariance estimator (subspace, projected covariance and rank) that simultaneously achieves $O(pk)$ per-client communication per round, where $p$ is the data feature dimension and $k$ rank of the low-rank approximation, and consists of summation operations at every aggregation stage, making it end-to-end compatible with secure aggregation. The method is a three-stage procedure: (i) a DP rank-selection step based on a shared random projection, (ii) a federated DP power method for the leading $k$ left singular vectors, and (iii) an Analyze Gauss mechanism applied to the projected $k \times k$ covariance. We provide an end-to-end error bound for the DP low-rank covariance estimate that separates cleanly into a subspace term and a projected covariance matrix noise term.

    \item \textbf{Convergence of low-rank preconditioned DP-FC.} We analyze federated DP-FC with the DP low-rank preconditioner on linear regression and show how to extend the analysis to generalized linear models (GLMs). We show that a rank-$k$ preconditioner is approximately a Newton step inside the estimated subspace and a gradient step on its orthogonal complement, and show that there is an effective condition number that governs the convergence and that it is
    commonly much smaller than the condition number governing the convergence of DP gradient descent (DP-GD). This reduces the number of communication rounds relative to first-order DP-FL by a large explicit factor and gives a closed-form step size. 

    \item \textbf{Empirical validation.} We experimentally show that our DP covariance approximation matches or improves upon the centralized $O(p^2)$ Analyze Gauss baseline at a fraction of the communication cost, and, unlike LDP-style federated baselines, its accuracy does not degrade as the number of clients grows. When used as a preconditioner for gradients in DP-FL fine-tuning, it improves the accuracy-communication trade-off over both vanilla federated DP-FC (with $O(p^2)$ communication) and first-order baselines on benchmark image datasets for moderate privacy budgets.
\end{itemize}

\section{The Problem}

Assume that each client $ j \in [N] \coloneqq \{1, \dots, N\}$ has a data matrix $ \mathbf{X}_j = (\mathbf{x}_{1}^{(j)},\ldots,\mathbf{x}_{n_j}^{(j)}) \in \mathbb{R}^{p \times n_j}$, where $ p $ denotes the dimension and $ n_j $ the number of samples at client $ j $. The full data matrix is denoted by $ \mathbf{X} = (\mathbf{X}_1, \dots, \mathbf{X}_N) \in \mathbb{R}^{p \times n} $, where $ n = \sum_{j=1}^N n_j $. 
The singular value decomposition (SVD) is defined by 
$
\mathbf{X} = \mathbf{U} \mathbf{\Sigma} \mathbf{W}^\top = \sum\nolimits_{i} \sigma_i \mathbf{u}_i \mathbf{w}_i^\top,
$
where the orthogonal matrices $ \mathbf{U} = (\mathbf{u}_1, \dots, \mathbf{u}_p) \in \mathbb{R}^{p \times
p} $ and $ \mathbf{W} = (\mathbf{w}_1, \dots, \mathbf{w}_n) \in \mathbb{R}^{n \times n} $ contain the left
and right singular vectors, respectively, and $ \mathbf{\Sigma} \in \mathbb{R}^{p \times
n} $ contains the singular values $ \sigma_i $, $ 1 \leq i \leq \min\{n, p\} $,
on its main diagonal.

In principal component analysis (PCA), one seeks a set of $ k $ mutually
orthogonal vectors called principal components that capture the directions of
largest variance of the data matrix. The first $k$ principal
components are given by the SVD as the $ k $ left singular vectors $ (\mathbf{u}_1,
\dots, \mathbf{u}_k) $ of the data matrix. They also coincide with the principal $k$ eigenvectors of the covariance matrix\footnote{Following standard usage in the DP-PCA literature~\citep[e.g.,][]{dwork2014analyze}, we refer to the (uncentered) second moment matrix $\mathbf{X}\mathbf{X}^\top$ as the covariance matrix.} 
$\mathbf{K} \coloneqq \mathbf{X}\mathbf{X}^\top = \mathbf{U} \boldsymbol{\Lambda} \mathbf{U}^\top$, where $\boldsymbol{\Lambda} \in \mathbb{R}^{p\times p}$ is a diagonal matrix containing the eigenvalues, $\lambda_i\geq 0$, $1\leq i\leq p$, on its diagonal. 
Notice that the global covariance matrix can also be written as a sum of the client-wise covariance matrices: 
$
\mathbf{X}\mathbf{X}^\top = \sum\nolimits_{j=1}^N \mathbf{X}_j \mathbf{X}_j^\top.
$

\para{Federated DP-PCA.}
Our first goal is to develop a differentially private federated algorithm to compute a global $ k $-dimensional subspace by aggregating contributions from the clients and also produce a global rank-$ k $ covariance matrix estimate.
In addition to approximating $\mathbf{K}$ accurately with DP guarantees, we want a small communication overhead. Na\"ive approaches require $ O(p^2) $ communication (since $ \mathbf{K} $ has $ p^2 $ elements), however, we would like to achieve $O(p \cdot k) $ communication, where $ k $ is the number of leading singular vectors to be estimated.

\para{Preconditioned federated DP optimization.}
Suppose in addition that each sample $\mathbf{x}$ has a label $y$, write $D$ for the full labeled dataset, and consider the regularized empirical risk
$
L(\theta, D) = \tfrac{1}{n} \sum\nolimits_{(\mathbf{x},y) \in D} \ell(\theta; \mathbf{x}, y) + \tfrac{\lambda}{2} \|\theta\|_2^2,
$
$\lambda > 0$, 
where $\ell$ is a generalized linear model (GLM) loss, e.g., a linear classification head trained on output features of a pretrained model. DP-FC~\citep{mehta2023differentially} is the preconditioned DP gradient descent
\begin{equation} \label{eq:dpfc_problem}
\theta_{j+1} = \theta_j - \eta \, \widehat{\mathbf{G}}^{-1} \big( \nabla L(\theta_j, D) + \mathbf{E}_j \big),
\end{equation}
where $\mathbf{E}_j$ is the Gaussian noise added to the sum of clipped per-sample gradients and $\widehat{\mathbf{G}}$ is a DP estimate of $\tfrac{1}{n}\mathbf{X}\mathbf{X}^\top + \lambda \mathbf{I}$. For GLMs the Hessian of $L$ is dominated by this matrix (up to diagonal scaling), so that~\eqref{eq:dpfc_problem} can be seen as an approximate Newton iteration. In particular, in the expression~\eqref{eq:dpfc_problem}, both the gradient $\nabla L(\theta_j, D)$ and the covariance $\mathbf{X}\mathbf{X}^\top$ are sums of client-wise contributions.

In DP-FC~\citep{mehta2023differentially}, $\widehat{\mathbf{G}}$ is obtained by Analyze Gauss applied to the full $p \times p$ covariance. The naïve federated version considered as a baseline method by~\citet{krouka2025communication} is straightforward: the clients upload and the server broadcasts $O(p^2)$ scalars. 
Our second goal is to remove the communication costs by considering the approximation
$$
\widehat{\mathbf{G}} = \mathbf{V}_k \widetilde{\mathbf{Y}}_k \mathbf{V}_k^\top + \lambda \mathbf{I},
$$
where $(\mathbf{V}_k, \widetilde{\mathbf{Y}}_k) \in \mathbb{R}^{p \times k} \times \mathbb{R}^{k \times k}$ is the factored output of a federated DP-PCA algorithm. Then only $O(p \cdot k)$ scalars per round are communicated, $\widehat{\mathbf{G}}^{-1}$ is applied in $O(pk + k^2)$ operations without forming any $p \times p$ matrix, and the preconditioner noise is a $k \times k$ matrix. We require that every stage (rank selection, basis computation, projected covariance estimation, gradient rounds) is a sum of client-wise contributions plus Gaussian noise, so that the whole pipeline is compatible with secure aggregation and its record-level $(\varepsilon,\delta)$-DP guarantee (under the add/remove relation of Section~\ref{sec:dp_prelims}) follows from a single composition of Gaussian mechanisms (Lemma~\ref{lem:gauss_dp}).

\section{Preliminaries of Differential Privacy}
\label{sec:dp_prelims}

DP gives quantifiable measures in terms of the parameters $\veps$ and $\delta$ that guarantee
indistinguishability of any individual element in the dataset from the output of the
algorithm. 

\para{Neighboring datasets.} Throughout, we use the record-level add/remove (unbounded) neighboring relation: two datasets $\mathbf{D}$ and $\mathbf{D}'$ are neighbors, denoted $\mathbf{D} \simeq \mathbf{D}'$, if one is obtained from the other by adding or removing a single record $(\mathbf{x},y)$, i.e., a single column of $\mathbf{X}$, held by any one of the clients. All the mechanisms considered can be seen as noisy sums of per-record contributions. Whenever a normalization is needed (e.g., $\tfrac{1}{n}\mathbf{X}\mathbf{X}^\top$ or an averaged gradient), the noisy sum is divided by the total number of data elements $n$, which we treat as public. Remark~\ref{rem:substitution} explains how the guarantees translate to the substitute-one relation.

\begin{defn}[\citealp{dworkOurDataOurselves2006}] 
A randomized algorithm (mechanism) $\mathcal M : \mathcal D \to \mathcal R$ is
$(\veps,\delta)$-differentially private if for all neighboring datasets $\mathbf{D} \simeq \mathbf{D}'$, $\mathbf{D},\mathbf{D}' \in \mathcal D$, and for all measurable sets $S$ in the output space $\mathcal R$, $S \subset \mathcal R$, we have
\[
    P(\mathcal M(\mathbf{D}) \in S) \leq e^\veps P(\mathcal M(\mathbf{D}') \in S) + \delta.
\]
\end{defn}

All the subroutines we consider in this work, as well as their combination, can be seen as adaptive compositions of Gaussian mechanisms. In its simplest form, a Gaussian mechanism $\mathcal{M}$ is of the form 
$$
    \mathcal{M}(\mathbf{D}) = \mathbf{f}(\mathbf{D}) + \mathbf{Z}, 
$$
where $\mathbf{f} \, : \, \mathcal{D} \rightarrow \mathbb{R}^p$ and $\mathbf{Z} \sim \mathcal{N}(\boldsymbol{0}, \sigma^2 \mathbf{I}_p)$, and $\Delta = \max_{\mathbf{D} \simeq \mathbf{D}'} \norm{\mathbf{f}(\mathbf{D}) - \mathbf{f}(\mathbf{D}')}_2$
gives the $L_2$-sensitivity. 
In case of a matrix-valued Gaussian mechanism, we add i.i.d. distributed noise to each element of the matrix and consider the vectorized $L_2$-sensitivity, i.e., the sensitivity in Frobenius norm, which for a matrix $\mathbf{A}=(A_{ij})$ is defined as $\norm{\mathbf{A}}_{F} = \sqrt{\text{tr}(\mathbf{A}^\top \mathbf{A})} = \sqrt{\sum_{i,j} A_{ij}^2}$.

DP guarantees for a composition of $T$ Gaussian mechanisms, each with varying sensitivity $\Delta_i$ and noise scale $\sigma_i$, is given by the following result.
\begin{lem} \label{lem:gauss_dp}
Suppose we have a composition of $T$ Gaussian mechanisms, where the $i$th mechanism has $L_2$-sensitivity $\Delta_i$ and noise scale parameter $\sigma_i$, $i \in [T]$. Then, the total composition is $(\veps,\delta)$-DP for
\begin{equation} \label{eq:gdp}
    \begin{aligned}
         \delta(\veps) = \Phi\left( - \frac{\veps}{\widetilde{\Delta}} + \frac{\widetilde{\Delta}}{2} \right) - e^\veps \Phi\left( - \frac{\veps}{\widetilde{\Delta}} - \frac{\widetilde{\Delta}}{2} \right),
    \end{aligned}
\end{equation}
where $\widetilde{\Delta} = \sqrt{ \sum\nolimits_{i=1}^T \Delta_i^2 / \sigma_i^2  }.$ 
\end{lem}
\begin{remark}
    Equivalently, in Lemma~\ref{lem:gauss_dp}, the composition is $\widetilde{\Delta}$-Gaussian Differentially Private (GDP)~\citep{dong2022gaussian}: the expression~\eqref{eq:gdp} is the $(\veps,\delta)$-characterization of $\mu$-GDP with $\mu = \widetilde{\Delta}$.
\end{remark}

\begin{remark}[DP guarantees for substitute-one relation] \label{rem:substitution}
All $\veps$-values reported in this paper can be converted to the substitute-one (bounded) neighborhood relation, where $\mathbf{D}'$ is obtained from $\mathbf{D}$ by replacing one record. For the matrix releases (rank selection algorithm, power iterations and the projected covariance release), we can use the fact that replacing $\mathbf{x}$ by $\mathbf{x}'$ with $\norm{\mathbf{x}}_2, \norm{\mathbf{x}'}_2 \leq C$ gives $\norm{\mathbf{x}\mathbf{x}^\top - \mathbf{x}'\mathbf{x}'^\top}_F^2 = \norm{\mathbf{x}}_2^4 + \norm{\mathbf{x}'}_2^4 - 2(\mathbf{x}^\top\mathbf{x}')^2 \leq 2C^4$, and that in case of projected covariance release, multiplication by matrices with orthonormal columns does not increase the Frobenius norm. Hence these sensitivities grow by a factor $\sqrt{2}$. For the sums of clipped per-sample gradients, the sensitivity grows by a factor $2$. Since Lemma~\ref{lem:gauss_dp} depends on the mechanisms only through the ratios $\Delta_i/\sigma_i$, the substitute-one ($\veps,\delta$)-guarantees are obtained by evaluating~\eqref{eq:gdp} with the noise scales $\sigma_i$ replaced by $\sigma_i/\sqrt{2}$ for the DP-PCA stages and by $\sigma_i/2$ for the gradient rounds.
\end{remark}

\section{Communication-Efficient DP-PCA} \label{sec:comm_eff_pca}
\subsection{Basis Computation with DP Power Method}


We next discuss computing a differentially private $k$-principal subspace of the data matrix $\mathbf{X}$ using the DP Power Iteration~\citep{nicolas2024differentially}. 
This method simply applies a noisy power iteration and then projects the matrix onto the set of matrices with orthonormal columns. This means that we first compute
$$
\wt{\mathbf{V}}_i = (\mathbf{X}\mathbf{X}^\top) \mathbf{V}_{i-1} + \mathbf{Z}_i,
$$
where $\mathbf{Z}_i \in \mathbb{R}^{p \times k}$, and each element of $\mathbf{Z}_i$ is independent and zero-mean normally distributed with variance $\sigma_1^2$, where $\sigma_1 = C^2 \sigma$ for a noise scale $\sigma>0$ and $C>0$ is a public bound on the 2-norms of the data vectors (columns) of $\mathbf{X}$, enforced by clipping the data as preprocessing. Then, the obtained $\wt{\mathbf{V}}_i$ is orthogonalized via the QR-decomposition:
$$
(\mathbf{V}_i, \mathbf{R}_i) = \text{QR}(\wt{\mathbf{V}}_i).
$$
A distributed version of this method is straightforward: each client $j$ simply computes at iteration $i$
$$
(\mathbf{X}_j \mathbf{X}_j^\top) \mathbf{V}_{i-1} + \mathbf{Z}_{i,j},
$$
where each element of $\mathbf{Z}_{i,j}$ is independent and normally distributed with variance $\sigma_1^2/N$, and then sends this result to the server, which carries out the QR-orthogonalization. 
To bound the sensitivity of the Gaussian mechanism at each step, we can use the fact that $\| \mathbf{x}\mathbf{x}^\top \mathbf{V} \|_F = \|\mathbf{x}\|_2 \|\mathbf{V}^\top \mathbf{x}\|_2 \leq \|\mathbf{x}\|_2^2
$ for any $\mathbf{V} \in \mathbb{R}^{p \times k}$ with orthonormal columns. Hence, under the add/remove relation, each iteration is a Gaussian mechanism with sensitivity $C^2$ and noise standard deviation $\sigma_1 = C^2\sigma$.

\subsection{Private Covariance Matrix Estimation}

After the basis approximation with DP power method, the local covariance matrices are projected and aggregated privately. Putting together,
the core of our proposed algorithm (preceded by the rank selection of Section~\ref{sec:initialization}) consists of the following two stages:

\textbf{First Stage:} Carry out the full batch DP power method with $T$ iterations to obtain an
orthogonal matrix $\mathbf{V}_k \in \mathbb{R}^{p \times k}$.

\textbf{Second Stage:} Project the global covariance matrix $\mathbf{X}\mathbf{X}^\top$
onto the subspace spanned by the columns of $\mathbf{V}_k$,
$\mathbf{Y}_k = \mathbf{V}_k^\top (\mathbf{X}\mathbf{X}^\top) \mathbf{V}_k$,
and release
$$
\wt{\mathbf{Y}}_k = \mathbf{Y}_k + \mathbf{Z},
$$
where $\mathbf{Z} \in \mathbb{R}^{k \times k}$ is a symmetric Gaussian matrix
with i.i.d.\ $\mathcal{N}(0, \sigma_2^2)$ upper-triangular entries (including
the diagonal), where $\sigma_2 = C^2 \sigma'$ for a noise scale $\sigma'>0$, following the Analyze Gauss mechanism~\citep{dwork2014analyze}. The scaling with $C^2$ comes from the fact that under the add/remove relation,
$\|\mathbf{V}_k^\top\mathbf{x}\mathbf{x}^\top\mathbf{V}_k\|_F = \|\mathbf{V}_k^\top\mathbf{x}\|_2^2 \leq C^2$,
and the upper-triangular vectorization has $\ell_2$-norm at most the Frobenius norm,
so this release is a Gaussian mechanism with sensitivity $C^2$ and noise standard deviation $\sigma_2$. Also this step is carried out in a distributed fashion with
the sum of local contributions equaling $\mathbf{Y}_k$ while off-loading the noise addition to the clients.

The output of these two stages is the factored pair $(\mathbf{V}_k, \wt{\mathbf{Y}}_k)$. Forming the rank-$k$ approximation $\wt{\mathbf{K}}_k := \mathbf{V}_k \wt{\mathbf{Y}}_k \mathbf{V}_k^\top$ can be performed locally, so the server never needs to transmit the full covariance matrix, keeping the communication at $O(pk)$ per round and $O(Tpk + k^2)$ in total.

\begin{remark}[Privacy of the two-stage algorithm] \label{rem:two_stage_gdp}
The two stages form an adaptive composition of $T+1$ Gaussian mechanisms, each with sensitivity $C^2$: $T$ mechanisms with noise standard deviation $\sigma_1$ and one with noise standard deviation $\sigma_2$. By Lemma~\ref{lem:gauss_dp}, the release $(\mathbf{V}_k, \wt{\mathbf{Y}}_k)$ is therefore $\mu$-GDP with
$$
\mu = \sqrt{ \frac{T C^4}{\sigma_1^2} + \frac{C^4}{\sigma_2^2} } = \sqrt{ \frac{T}{\sigma^2} + \frac{1}{\sigma'^2} },
$$
and its $(\veps,\delta)$-guarantees are given by~\eqref{eq:gdp} with $\widetilde{\Delta} = \mu$. In particular, the privacy guarantee does not depend on the number of clients $N$.
\end{remark}

\subsection{Error Bound for Covariance Estimation}

Write $\mathbf{K} = \mathbf{X}\mathbf{X}^\top$ and $\widetilde{\mathbf{K}}_k = \mathbf{V}_k \widetilde{\mathbf{Y}}_k \mathbf{V}_k^\top$ for the quantities of Section~\ref{sec:comm_eff_pca}. The error $\mathrm{Err}_k \;=\; \mathbf{K} - \widetilde{\mathbf{K}}_k$ decomposes into two parts: a subspace error, arising from the approximate basis $\mathbf{V}_k$ of the first stage, and a noise error, arising from the symmetric Gaussian perturbation $\mathbf{Z}$ in the second stage. The basis error is bounded in Theorem~\ref{thm:ref_bound} and is governed by the spectral gap $\lambda_k - \lambda_{k+1}$ of the true covariance. The second error is essentially the spectral norm of a $k \times k$ symmetric
Gaussian matrix (Lemma~\ref{thm:symmetric_concentration}) and it scales linearly in $\sigma_2\sqrt{k}$ rather than $\sigma_2\sqrt{p}$. Combining the two bounds via the triangle inequality and Lemma~\ref{lem:triangle_PAP} gives the following.

\begin{thm} \label{thm:error_two_phases}
Let $\mathbf{K} = \mathbf{X}\mathbf{X}^\top \in \mathbb{R}^{p \times p}$ be the true covariance with eigenvalues $\lambda_1 \geq \cdots \geq \lambda_p \geq 0$, let $C$ be the public bound on $\|\mathbf{x}_i\|_2$, and
let $\mathbf{Q} \in \mathbb{R}^{p \times k}$ be the orthonormal basis
returned by the basis computation stage (Algorithm~\ref{alg:DP_power} in the appendix) run for a public number of iterations $T$ with i.i.d.\ noise entries distributed as $\mathcal{N}(0,\sigma_1^2)$. Assume the second stage releases
$\widetilde{\mathbf{Y}}_k = \mathbf{V}_k^\top \mathbf{K} \mathbf{V}_k + \mathbf{Z}$,
where $\mathbf{Z} \in \mathbb{R}^{k \times k}$ is a symmetric Gaussian
matrix whose upper-triangular entries (including the diagonal) are
i.i.d.\ $\mathcal{N}(0,\sigma_2^2)$ and whose lower-triangular entries
are copied from the upper triangle. Define the total error
$
\mathrm{Err}_k \;:=\; \mathbf{K} - \mathbf{V}_k\widetilde{\mathbf{Y}}_k \mathbf{V}_k^\top.
$
\begin{enumerate}
\item[(i)] Let $\varepsilon_1 + \varepsilon_2 = \varepsilon$ and $\delta_1 + \delta_2 = \delta$ with $0 < \varepsilon_i \leq 2\log(1/\delta_i)$, $i=1,2$ (e.g., $\varepsilon_i = \varepsilon/2$, $\delta_i = \delta/2$), and let
$
\sigma_1 = \varepsilon_1^{-1} C^2 \sqrt{4 T \log(1/\delta_1)}$ and $\sigma_2 = \varepsilon_2^{-1} C^2 \sqrt{4 \log(1/\delta_2)}.
$
Then the release $( \mathbf{V}_k, \widetilde{\mathbf{Y}}_k)$ is $(\varepsilon,\delta)$-DP.
\item[(ii)] Fix a constant $\tau \geq 1$ and assume $T \geq L$ for the number of iterations
$L = O\!\left(\tfrac{\lambda_k}{\lambda_k - \lambda_{k+1}}\log p\right)$ required in Theorem~\ref{thm:ref_bound}. Then,
for any $\rho, \nu \in (0,1)$, with probability at least
$1 - \nu - \rho - p_0$, where $p_0 = \tau^{-\Omega(1)} + e^{-\Omega(p)}$, we have that
\begin{equation*}
    \|\mathrm{Err}_k\|_2
 \;\le\;
\sqrt{2}\bigl(\lambda_{k+1} + \lambda_1 \zeta\bigr)
  \;+\;
\sigma_2\Bigl(2\sqrt{k} + 2\sqrt{\log(2/\rho)}\Bigr),
\end{equation*}
where
\begin{equation*}
\zeta = \min\left\{1,\, O\!\left(
\frac{\sigma_1\sqrt{p}\bigl(1 + \sqrt{\log(T/\nu)}\bigr)}
     {\lambda_{k}-\lambda_{k+1}}
\cdot k
\right)\right\}.
\end{equation*}
\end{enumerate}
\end{thm}

This result separates into contributions of the two phases, which may guide the privacy budget allocation: one chooses $(\varepsilon_1, \varepsilon_2)$ subject to $\varepsilon_1 + \varepsilon_2 = \varepsilon$ so as to balance the two error sources. A natural approach to minimize the total error is to equate the two terms. This puts a larger share of the privacy budget into $\varepsilon_1$ when the spectral gap $\lambda_k - \lambda_{k+1}$ is small. 

The two terms also show the dimensional cost of privacy. The first term has the $\sqrt{p}$-factor from the norm bound on the $p \times k$ Gaussian noise injected at each iteration of noisy power method. The Analyze Gauss baseline applied directly to the $p \times p$ covariance incurs an error of order $O(\sigma \sqrt{p})$~\citep{dwork2014analyze}. Our bound replaces this by a gap-dependent subspace term and a noise term scaling with $\sqrt{k}$. Its main benefit is that it is achieved with $O(Tpk+k^2)$ rather than $O(p^2)$ communication cost in a federated setting.

\section{Private Rank Estimation} \label{sec:initialization}
We also propose a lightweight federated initialization procedure that produces a DP estimate of an effective rank of the global covariance matrix $\mathbf{X}\mathbf{X}^\top = \sum_{i} \mathbf{X}_i\mathbf{X}_i^\top$. First, we compress $\mathbf{X}\mathbf{X}^\top$ via a shared Gaussian random projection to an $r\times r$ matrix that clients can communicate cheaply, clip the 2-norm of the projected samples, and release a noisy version of this projected matrix. The rank is then chosen by thresholding its eigenvalues.

We assume each client holds local data $\mathbf{X}_i\in\mathbb{R}^{p\times n_i}$ whose columns have bounded $\ell_2$-norm $\|\mathbf{x}\|_2\le C$, and that all clients share a common public seed.
For a clipping bound $\bar C>0$ (e.g.\ $\bar C = C$), let $\mathrm{clip}_{\bar C}(\mathbf{y}) = \mathbf{y}\cdot\min\{1,\bar C/\|\mathbf{y}\|_2\}$ for a vector $\mathbf{y}$. For a matrix argument, $\mathrm{clip}_{\bar C}$ is applied column-wise, i.e., $\mathrm{clip}_{\bar C}(\mathbf{Y}) = \big(\mathrm{clip}_{\bar C}(\mathbf{y}_1), \ldots, \mathrm{clip}_{\bar C}(\mathbf{y}_m)\big)$ for $\mathbf{Y} = (\mathbf{y}_1, \ldots, \mathbf{y}_m)$, so that every column of the result has $\ell_2$-norm at most $\bar C$. The procedure is given as follows:

\begin{enumerate}
    \item The server distributes a public seed to all clients.
    \item Each client uses the seed to generate the projection matrix
          $\mathbf{G}\in\mathbb{R}^{p\times r}$ with
          $G_{ij}\sim\mathcal{N}(0,1/r)$ i.i.d.
    \item Each client projects and clips (column-wise) its local data,
          $\wt{\mathbf{X}}_i = \mathrm{clip}_{\bar C}(\mathbf{G}^\top \mathbf{X}_i)\in\mathbb{R}^{r\times n_i}$,
          and sends 
          $\wt{\mathbf{X}}_i\wt{\mathbf{X}}_i^\top\in\mathbb{R}^{r\times r}$
          to the server.
    \item Via secure aggregation of the clients' noisy shares (see below), the server obtains
          $\sum_i\wt{\mathbf{X}}_i\wt{\mathbf{X}}_i^\top + \sigma\,\mathbf{E}
           =\wt{\mathbf{X}}\wt{\mathbf{X}}^\top + \sigma\,\mathbf{E}$, $\sigma>0$, where $\mathbf{E}$ is symmetric with i.i.d.\ $\mathcal{N}(0,1)$ entries on and above the diagonal, and
          estimates the rank from the eigenvalues of
          $\wt{\mathbf{C}}=\wt{\mathbf{X}}\wt{\mathbf{X}}^\top+\sigma\, \mathbf{E}$.
\end{enumerate}

In case of an honest-but-curious server and secure aggregation with distributed Gaussian noise we obtain the centralized DP guarantees. Each of the $N$ clients independently samples $\mathbf{E}_i$ with the same distribution and adds the share $\tfrac{\sigma}{\sqrt{N}}\mathbf{E}_i$ to its
local upload. The per-client communication of this initialization stage is only $O(r^2)$
scalars per round.
We next establish the DP
guarantee of the released matrix $\wt{\mathbf{C}}$.

\begin{lem}\label{lem:dp_initialization}
The DP guarantees of the mechanism
$$
\mathcal{M}(\mathbf{X})\;=\;\wt{\mathbf{X}}\wt{\mathbf{X}}^\top+\sigma\,\mathbf{E},
\qquad \wt{\mathbf{X}} = \mathrm{clip}_{\bar C}(\mathbf{G}^\top\mathbf{X})
$$
are those of the Gaussian mechanism with sensitivity $\bar C^2$ and noise level $\sigma$. In particular, $\mathcal{M}$ is $\mu$-GDP with $\mu = \bar C^2/\sigma$.
\end{lem}

Compared to releasing the unclipped projection $\mathbf{G}^\top\mathbf{X}\mathbf{X}^\top\mathbf{G}$, whose sensitivity can only be bounded with high probability over $\mathbf{G}$, clipping gives a smaller, deterministic sensitivity and a clean GDP privacy guarantee. Since $\|\mathbf{G}^\top\mathbf{x}\|_2^2$ concentrates around $\|\mathbf{x}\|_2^2$, the clipping perturbs the projected samples only mildly, and it can only decrease the eigenvalues, as $\wt{\mathbf{X}}\wt{\mathbf{X}}^\top\preceq\mathbf{G}^\top\mathbf{X}\mathbf{X}^\top\mathbf{G}$. The following result shows that in case we select the rank as the number of eigenvalues of $\wt{\mathbf{C}}$ exceeding the noise level, then with h.p. the selected rank is sandwiched between the numbers of eigenvalues of the clipped projected covariance $\wt{\mathbf{X}}\wt{\mathbf{X}}^\top$ above $2\sigma b_r$ and above zero; in particular, we do not overestimate the rank.

\begin{lem} \label{lem:rank_no_fp}
Let $\rho\in(0,1)$ and
$b_r\;=\;2\sqrt{r}+2\sqrt{\log(2/\rho)}$, and let
$\wh k = \big|\{\, i : \lambda_i(\wt{\mathbf{C}}) > \sigma\,b_r \,\}\big|$.
Then, with probability at least $1-\rho$,
$$
\big|\{ i : \lambda_i(\wt{\mathbf{X}}\wt{\mathbf{X}}^\top) > 2\sigma b_r\}\big|
\;\le\; \wh k \;\le\;
\big|\{ i : \lambda_i(\wt{\mathbf{X}}\wt{\mathbf{X}}^\top) > 0\}\big| \;\le\; \operatorname{rank}(\mathbf{X}\mathbf{X}^\top).
$$
\end{lem}

The accuracy of the proposed algorithm is illustrated empirically in Section~\ref{sec:experiments}.

\section{Convergence Analysis}

\label{sec:convergence}

We next analyze when the low-rank estimate $(\mathbf{V},\wt{\mathbf{Y}})$ makes
federated DP-FC preferable to federated DP-GD. To this end, we analyze linear regression which reveals the benefits of preconditioning most clearly, and comment on GLMs at the end. In short: for the same privacy budget, rank-$k$ preconditioning reaches the utility of DP-GD in a fraction of the rounds. Thereby, we show that the benefits of private second-order optimization shown by~\citet{ganesh2023faster} can be retained with DP low-rank approximation of the covariance. We consider a fixed-point analysis with a weighted norm, similarly to the alternating direction method of multipliers analysis by~\citet{cyffers2023noisy}.  As in~\citet{brown2024}, we write the model state parameter obtained via DP iterations as a closed-form expression with a deterministic optimization term and a weighted sum of injected DP noise.

\para{Setting.}
Throughout, $\lambda>0$ is the regularization parameter.
The values $\lambda_1\ge\dots\ge\lambda_p\ge0$ are the eigenvalues of
$\mathbf{C}:=\tfrac1n\mathbf{K}=\tfrac1n\mathbf{X}\mathbf{X}^\top=\mathbf{U}\boldsymbol{\Lambda}\mathbf{U}^\top$
(in this section, $\boldsymbol{\Lambda}$ and $\lambda_i$ thus refer to $\mathbf{C}$ rather than $\mathbf{K}$),
$\mathbf{U}_k$ holds the leading $k<p$ eigenvectors, and
$\mathbf{C}_\lambda:=\mathbf{C}+\lambda\mathbf{I}$. Consider
$$
\mathcal{L}(\theta)=\frac{1}{2n}\sum\nolimits_{(x,y)\in D}(\theta^\top x-y)^2+\tfrac{\lambda}{2}\norm{\theta}_2^2
$$
with minimizer $\theta_*$. Since $\mathcal{L}$ is quadratic with constant Hessian $\nabla^2\mathcal{L} = \mathbf{C}_\lambda$ and $\nabla\mathcal{L}(\theta_*) = 0$, we have $\nabla\mathcal{L}(\theta) = \mathbf{C}_\lambda(\theta-\theta_*)$ and
$$
\mathcal{L}(\theta)-\mathcal{L}(\theta_*)=\tfrac12\norm{\theta-\theta_*}_{\mathbf{C}_\lambda}^2,
$$
where for a symmetric positive definite (SPD) $\mathbf{M}$ we denote $\norm{\mathbf{v}}_{\mathbf{M}}^2 := \mathbf{v}^\top \mathbf{M}\mathbf{v}$.
For an SPD $\mathbf{P}$ we consider
\begin{equation}\label{eq:pdpgd}
    \theta_{j+1}=\theta_j-\eta\,\mathbf{P}^{-1}\bigl(\nabla\mathcal{L}(\theta_j)+\mathcal{E}_j\bigr),
    \quad \mathcal{E}_j\sim\mathcal{N}(0,\sigma_g^2\mathbf{I}).
\end{equation}
DP-GD iteration is given by $\mathbf{P}=\mathbf{I}$. DP-FC with our estimate is given by
\begin{equation}\label{eq:Ghat}
    \mathbf{P}=\wh{\mathbf{G}}:=\mathbf{V}\wt{\mathbf{Y}}\mathbf{V}^\top+\lambda\mathbf{I},
    \qquad \wt{\mathbf{Y}}=\mathbf{V}^\top\mathbf{C}\mathbf{V}+\mathbf{Z},
\end{equation}
where $\mathbf{Z}$ is symmetric with i.i.d.\ $\mathcal{N}(0,\sigma_2^2)$
upper-triangular entries ($\sigma_2 = C^2\sigma'/n$ with $C$ the record-norm bound; the factor $1/n$ is due to the normalization of $\mathbf{C}$, cf.\ Algorithm~\ref{alg:fed-analyze-gauss}). 
If
$\Delta_g$ is the $L_2$-sensitivity of the clipped averaged gradient (under the add/remove relation, $\Delta_g = c/n$ for per-sample gradients clipped to norm $c$ and normalization by the public $n$), then $T$
rounds are $\mu_g$-GDP with $\mu_g=\sqrt{T}\Delta_g/\sigma_g$ (Lemma~\ref{lem:gauss_dp}).
The use of the preconditioner in the iteration is post-processing.
For simplicity, we assume the gradients stay bounded in 2-norm, i.e., we ignore the clipping bias. We use the two
condition numbers
\begin{equation}\label{eq:kappa_k}
    \kappa_0:=\frac{\lambda_1+\lambda}{\lambda_p+\lambda},
    \qquad
    \kappa_k:=1+\frac{\lambda_{k+1}}{\lambda}.
\end{equation}

For symmetric $\mathbf{C}_\lambda$ and $\mathbf{G}\succ0$, the generalized eigenvalues of
$(\mathbf{C}_\lambda,\mathbf{G})$ are the scalars $\mu$ with $\mathbf{C}_\lambda\mathbf{x}=\mu\mathbf{G}\mathbf{x}$
for some $\mathbf{x}\neq\mathbf{0}$, i.e., the eigenvalues of $\mathbf{G}^{-1/2}\mathbf{C}_\lambda\mathbf{G}^{-1/2}$.

For the oracle preconditioner
$\mathbf{G}_\star:=\mathbf{U}_k\boldsymbol{\Lambda}_k\mathbf{U}_k^\top+\lambda\mathbf{I}$
(i.e., $\mathbf{V}=\mathbf{U}_k$, $\mathbf{Z}=\mathbf{0}$), the generalized eigenvalues of $(\mathbf{C}_\lambda,\mathbf{G}_\star)$ are $1$ for $i\le k$ and $1+\lambda_i/\lambda$ for $i>k$, and hence lie in $[1,\kappa_k]$ (Lemma~\ref{lem:woodbury}, Appendix~\ref{app:convergence_proofs}). Moreover, whenever $\wt{\mathbf{Y}}+\lambda\mathbf{I}_k\succ0$,
\begin{equation}\label{eq:Ghatinv}
    \wh{\mathbf{G}}^{-1}
    =\mathbf{V}\bigl(\wt{\mathbf{Y}}+\lambda\mathbf{I}_k\bigr)^{-1}\mathbf{V}^\top
    +\tfrac1\lambda\bigl(\mathbf{I}-\mathbf{V}\mathbf{V}^\top\bigr),
\end{equation}
so a preconditioned step is an approximate Newton step within the span of $\mathbf{V}$ and a gradient step with step size $\eta/\lambda$ on its orthogonal complement.

We state the utility bounds in terms of the factor $\gamma$, as defined and bounded in the following Lemma.

\begin{lem} \label{lem:perturb}
Let $$
\gamma:=\norm{\mathbf{G}_\star^{-1/2}\wh{\mathbf{G}}\mathbf{G}_\star^{-1/2}-\mathbf{I}}_2.
$$
Then 
$$
(1-\gamma)\mathbf{G}_\star\preccurlyeq\wh{\mathbf{G}}\preccurlyeq(1+\gamma)\mathbf{G}_\star.
$$
If $\gamma<1$, the generalized eigenvalues of $(\mathbf{C}_\lambda,\wh{\mathbf{G}})$ lie in
$\bigl[\tfrac{1}{1+\gamma},\tfrac{\kappa_k}{1-\gamma}\bigr]$, and
$(1-\gamma)\tr(\wh{\mathbf{G}}^{-1})\le\tr(\mathbf{G}_\star^{-1})$. Moreover,
with $\zeta:=\norm{(\mathbf{I}-\mathbf{V}\mathbf{V}^\top)\mathbf{U}_k}_2<1$
(bounded by Theorem~\ref{thm:ref_bound}, with the eigenvalues $n\lambda_i$ of $\mathbf{K}$),
\begin{equation*}
    \gamma\le\frac{2\lambda_1\zeta+\norm{\mathbf{Z}}_2}{\lambda},
\end{equation*}
and with probability at least $1-\rho$,
$\norm{\mathbf{Z}}_2\le\sigma_2\bigl(2\sqrt{k}+2\sqrt{\log(2/\rho)}\bigr)$.
\end{lem}
We write $\wh\kappa_k:=\kappa_k\frac{1+\gamma}{1-\gamma}$ for the effective
condition number of the private preconditioner.
\begin{thm}[Utility of preconditioned DP-GD]\label{thm:total}
Let $\mathbf{P}\succ0$ be independent of $(\mathcal{E}_j)$, let the generalized
eigenvalues of $(\mathbf{C}_\lambda,\mathbf{P})$ lie in $[a,b]$, and let $0<\eta\le1/b$.
Then
\begin{equation}\label{eq:total}
\begin{aligned}
     \mathbb{E}\,\mathcal{L}(\theta_T)-\mathcal{L}(\theta_*) 
\;\le\; (1-\eta a)^{2T}\bigl(\mathcal{L}(\theta_0)-\mathcal{L}(\theta_*)\bigr)
    +\frac{\eta\,\sigma_g^2}{2}\,\tr(\mathbf{P}^{-1}).
\end{aligned}
\end{equation}
\end{thm}

Take $\eta=1/b$. DP-GD has $a=\lambda_p+\lambda$, $b=\lambda_1+\lambda$, and
$\tr(\mathbf{I})=p$. This gives rate $1-1/\kappa_0$ and per-round noise
coefficient $p/(\lambda_1+\lambda)$. By Lemma~\ref{lem:perturb}, our
preconditioner gives rate $1-1/\wh\kappa_k$ and noise coefficient at most
$\tr(\mathbf{G}_\star^{-1})/\kappa_k\le p/(\lambda+\lambda_{k+1})$. The
preconditioned method takes larger steps and so injects more noise per round,
but it needs proportionally fewer rounds. 
The next corollary studies these trade-offs and gives an explicit comparison of the low-rank preconditioned DP-GD and DP-GD.

\begin{cor}\label{cor:compare}
Let $\gamma<1$ and let the $T$ gradient rounds be $\mu_g$-GDP in total, i.e.\
$\sigma_g^2=T\Delta_g^2/\mu_g^2$, and assume $R_0 \geq \Delta_g^2 N/(4\mu_g^2)$ for $N \in \{N_{\mathrm{GD}}, N_k\}$. Write
$R_0:=\mathcal{L}(\theta_0)-\mathcal{L}(\theta_*)$,
\begin{equation*}
\begin{aligned}
     N_{\mathrm{GD}}:=\frac{p}{\lambda_p+\lambda},\quad
    N_k:=\frac{1+\gamma}{1-\gamma}\Bigl[\sum_{i\le k}\frac{1}{\lambda_i+\lambda}+\frac{p-k}{\lambda}\Bigr],     \quad \Phi(N):=\frac{\Delta_g^2N}{4\mu_g^2}\Bigl(3+\log\frac{4\mu_g^2R_0}{\Delta_g^2N}\Bigr).
\end{aligned} 
\end{equation*}
\begin{enumerate}
\item[(i)] With $\eta = 1/b$, i.e., $\eta = 1/(\lambda_1+\lambda)$ for DP-GD and $\eta = (1-\gamma)/\kappa_k$ for DP-FC, DP-GD after
$T_{\mathrm{GD}}=\bigl\lceil\tfrac{\kappa_0}{2}\log\tfrac{4\mu_g^2R_0}{\Delta_g^2N_{\mathrm{GD}}}\bigr\rceil$
rounds, and DP-FC with $\wh{\mathbf{G}}$ after
$T_k=\bigl\lceil\tfrac{\wh\kappa_k}{2}\log\tfrac{4\mu_g^2R_0}{\Delta_g^2N_k}\bigr\rceil$
rounds, satisfy $\mathbb{E}\,\mathcal{L}(\theta_T)-\mathcal{L}(\theta_*)\le\Phi(N_{\mathrm{GD}})$
and $\le\Phi(N_k)$, respectively.
\item[(ii)] $N_k\le\frac{1+\gamma}{1-\gamma}\bigl(1+\frac{\lambda_p}{\lambda}\bigr)N_{\mathrm{GD}}$.
For $\gamma=\lambda_p=0$, $N_k=N_{\mathrm{GD}}-\sum_{i\le k}\frac{\lambda_i}{\lambda(\lambda_i+\lambda)}$.
\end{enumerate}
\end{cor}
When neglecting the log terms, we can read Corollary~\ref{cor:compare} such that at equal privacy, the low-rank preconditioned method matches the privacy--utility bound of DP-GD up to the factor $\frac{1+\gamma}{1-\gamma}(1+\lambda_p/\lambda)$, and slightly improves on it when $\gamma$ and $\lambda_p/\lambda$ are small. It does so in a fraction $\approx\wh\kappa_k/\kappa_0$ of the
rounds, where
\begin{equation*}
    \frac{\kappa_0}{\kappa_k}=\frac{\lambda_1+\lambda}{\lambda_p+\lambda}\cdot\frac{\lambda}{\lambda+\lambda_{k+1}}.
\end{equation*}
The gain is therefore essentially in rounds, which is the
expensive unit in cross-silo FL (Section~\ref{subsec:dpfc}). We remark that this simplified analysis makes three shortcuts. First, "equal privacy" refers to an equal GDP-budget of the gradient descent phase $\mu_g$: our method and DP-FC also spend the budget $\mu_{\mathrm{PCA}}$ on the preconditioner, so under a fixed total budget $\mu$ they have $\mu_g^2 = \mu^2 - \mu_{\mathrm{PCA}}^2$, whereas DP-GD uses $\mu_g = \mu$. Second, $T_k$ excludes the $T+1$ rounds used to construct $\wh{\mathbf{G}}$. Third, $\gamma$ and $\lambda_{k+1}$ are unknown in practice, so $\eta$ is tuned. The conclusion of gaining similar privacy-utility trade-off as DP-GD in fewer training rounds also matches with the conclusion made by~\citet{ganesh2023faster} on the benefits of private second-order optimization.

%
\para{Extension to GLMs.} 
For a convex GLM loss, the mean value theorem gives
$\nabla\mathcal{L}(\theta_j)=(\bar{\mathbf{H}}_j+\lambda\mathbf{I})(\theta_j-\theta_*)$
with $\bar{\mathbf{H}}_j=\tfrac1n\sum\ell''(\xi_x)\,xx^\top$. Suppose that, after
rescaling, $\ell''\in[c,1]$ along the iterate path, as holds for the logistic loss
with bounded logits. Then
$c\mathbf{C}+\lambda\mathbf{I}\preccurlyeq\bar{\mathbf{H}}_j+\lambda\mathbf{I}\preccurlyeq\mathbf{C}_\lambda$,
so the generalized eigenvalues of $(\bar{\mathbf{H}}_j+\lambda\mathbf{I},\wh{\mathbf{G}})$
lie in $[c/(1+\gamma),\kappa_k/(1-\gamma)]$. The per-step contraction in the
$\wh{\mathbf{G}}$-norm is then $1-c/\wh\kappa_k$; the analogous rate for DP-GD is
$1-c/\kappa_0$. Since $\bar{\mathbf{H}}_j$ varies with $j$, the noise term is
bounded with the fixed-point argument of \citet{cyffers2023noisy} in the
$\wh{\mathbf{G}}$-norm (Appendix~\ref{app:glm}). The round comparison of
Corollary~\ref{cor:compare} carries over with $\kappa_0,\wh\kappa_k$ replaced by
$\kappa_0/c,\wh\kappa_k/c$. The noise term of the GLM bound (Appendix~\ref{app:glm}) is, however, looser than that of Theorem~\ref{thm:total} by a factor of order $\wh\kappa_k/c$, so Corollary~\ref{cor:compare}(ii) does not carry over directly. Finally, we remark that the analysis is for scalar-output GLMs and directly generalizable for the multiclass softmax loss.

\section{Experiments}
\label{sec:experiments}

We evaluate our method on standard benchmark datasets and compare it against the relevant baselines.
In all figures, results are means over 5 independent runs, and error bars denote $\pm 1$ empirical standard deviation.

\subsection{Illustration of Rank Estimation}
\label{subsec:exp_dwork}

Figure~\ref{fig:init000} illustrates our DP rank estimation procedure on
MNIST~\citep{lecun1998gradient} ($p=784$) at the privacy level
$(\varepsilon,\delta)=(2.0,\,10^{-5})$, with the noise scale $\sigma$
calibrated according to Lemma~\ref{lem:dp_initialization}. We plot the
square roots of the eigenvalues of the noisy projected $r\times r$ matrix $\widetilde{\mathbf{C}}$ for three projection dimensions $r\in\{20,50,80\}$, together with the true singular values of $\mathbf{X}$ for reference. At all three values of $r$, the dominant part of the spectrum is recovered accurately.

\begin{figure}[h!]
\centering
\includegraphics[width=0.5\columnwidth]{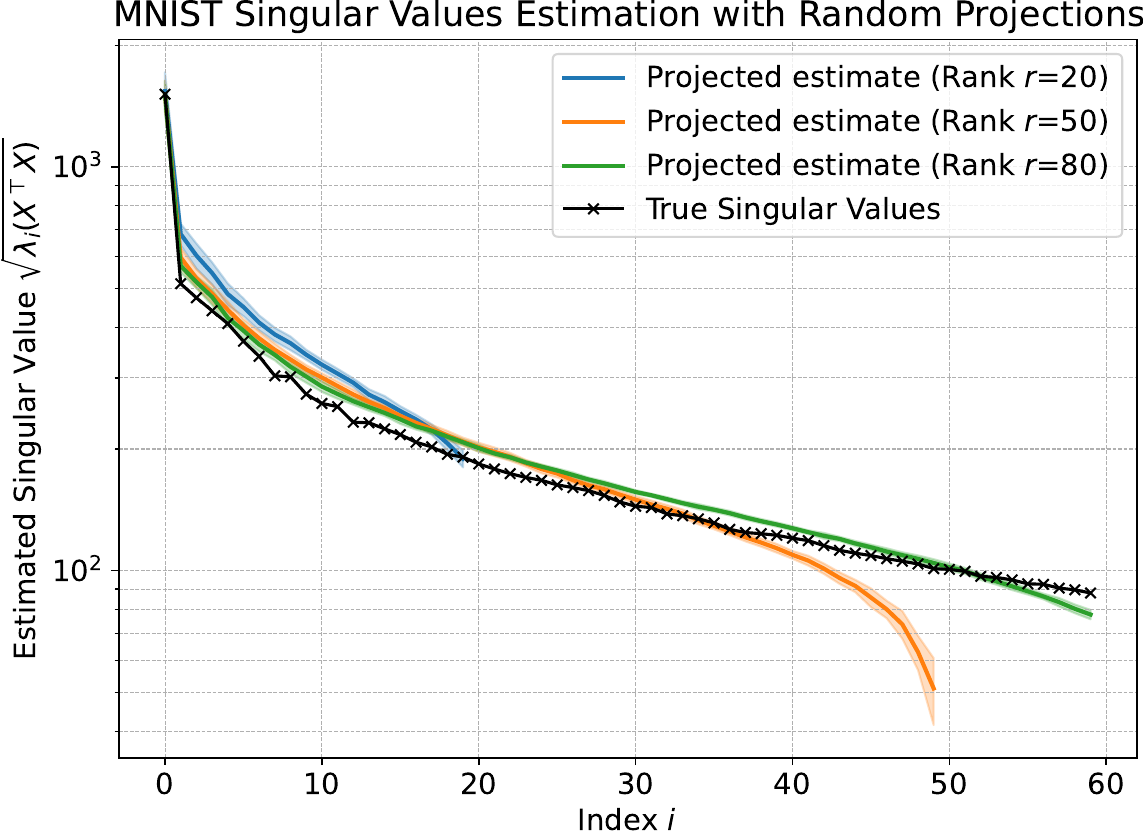}
\caption{Singular values of the MNIST feature matrix $\mathbf{X}\in\mathbb{R}^{p\times n}$ versus the square roots of the eigenvalues of the DP estimate $\widetilde{\mathbf{C}}$ for projection dimensions $r\in\{20,50,80\}$.}
\label{fig:init000}
\end{figure}
\subsection{DP-PCA Comparison}
\label{subsec:exp_dwork2}
We compare the two-stage algorithm of Section~\ref{sec:comm_eff_pca} (basis computation followed by the projected covariance release, with $k$ fixed) to the central DP-PCA mechanism Analyze Gauss~\citep{dwork2014analyze}.
Figure~\ref{fig:uci_har} reports the relative Frobenius error
$\norm{\mathbf{K}_k - \widetilde{\mathbf{K}}_k}_F\,/\,\norm{\mathbf{K}_k}_F$, where $\mathbf{K}_k$ is the best rank-$k$ approximation of $K$ in 2-norm, as a function of the privacy budget $\varepsilon$, for $k \in \{10, 20, 40\}$, on preprocessed UCI HAR features ($p=561$). We see that our method improves upon the $O(p^2)$ baseline at every $\varepsilon$ and $k$, with the largest gains at small $\varepsilon$, while requiring only $O(T p k + k^2)$ communication.
We use $T=3$ power iterations.
Appendix~\ref{sec:experiment_details} shows the same behavior on SVHN ($p=3072$), and also compares against the locally private federated method of~\citet{grammenos2020federated}, whose error grows with the number of clients while ours does not depend on it.
\begin{figure}[h!]
\centering
\includegraphics[width=0.5\columnwidth]{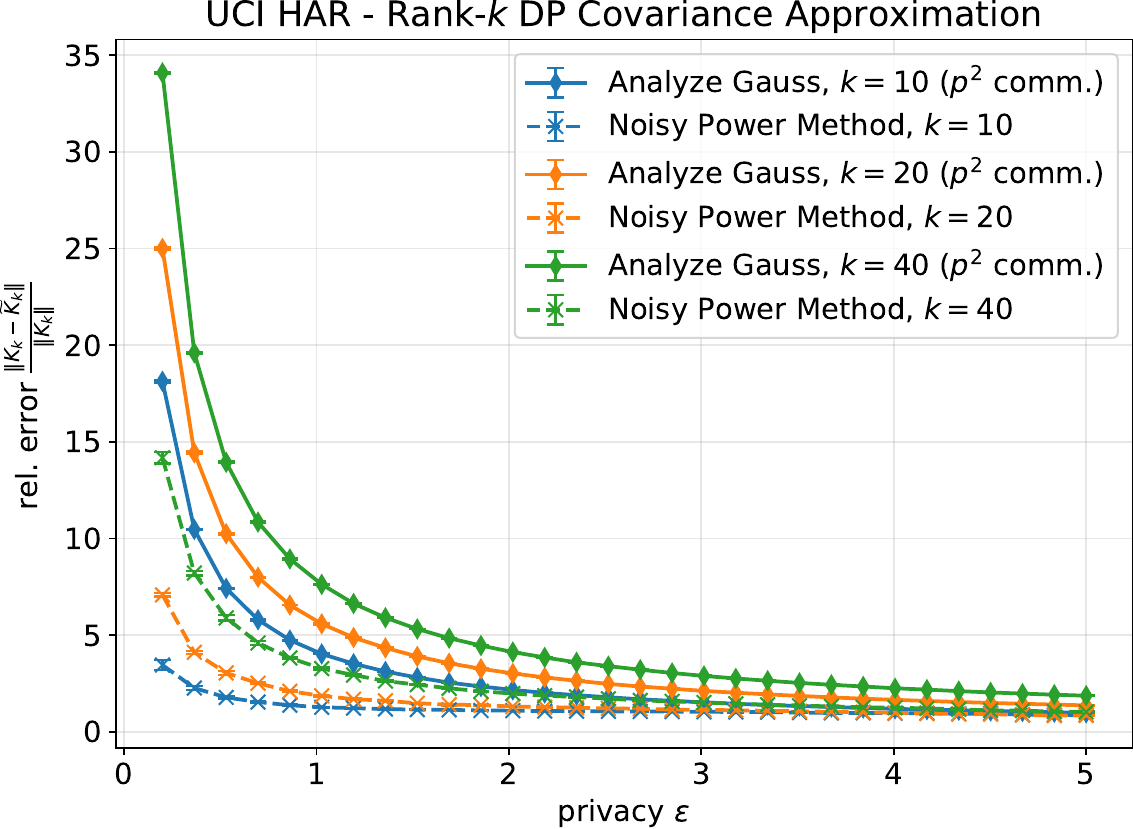}
\caption{Comparison of our method against the central DP-PCA method of~\citet{dwork2014analyze} on UCI HAR ($p=561$) with target subspace dimension $k \in \{10, 20, 40\}$. The $y$-axis shows the relative Frobenius error of the rank-$k$ covariance approximation, the $x$-axis $\varepsilon$ for $\delta=10^{-5}$. The noisy power method communicates $Tpk+k(k+1)/2$ scalars in total (from 16885 to 68140 for $T=3$), whereas Analyze Gauss communicates the upper triangle of the covariance, $p(p+1)/2 = 157641$.}
\label{fig:uci_har}
\end{figure}

\subsection{Preconditioned DP Fine-Tuning}
\label{subsec:dpfc}

We train a linear classification head under DP on frozen features: ResNet-34 features ($p=512$) for FashionMNIST and ViT-B/16 features ($p=768$) for CIFAR-10 and CIFAR-100. For the low-rank approximation of the covariance, we use a fixed rank $k = \ceil{1.2\,m}$, where $m$ is the number of classes, and $T = 2$ DP power iterations. We try to keep both $k$ and $T$ small so that the cost of constructing the preconditioner stays low relative to the gradient rounds. 
We fix $k$ and do not run the rank-selection step of Section~\ref{sec:initialization}; running it would add the term $\bar C^4/\sigma^2$ to $\mu^2$ in the composition. Features are extracted with ImageNet 1k-pretrained models obtained from Torchvision~\citep{marcel2010torchvision}. All methods satisfy the same total $(\varepsilon,\delta)$-DP guarantee; for the preconditioned methods, part of the budget is spent on the covariance estimate.
We compare federated DP-FC equipped with our covariance estimate (Our Method) against vanilla federated DP-FC, which differs only in applying Analyze Gauss to the full covariance, and against the DP first-order baselines DP-GD, DP-GD with momentum, and DP-Adam.
We omit DP-FedNew~\citep{krouka2025communication}: as noted by its authors, its record-level sensitivity bound is loose and the method is most competitive under user-level DP, and in preliminary record-level experiments it was not competitive with the second-order baselines above.
The main text reports $\varepsilon=2$; results for $\varepsilon\in\{1,2,4,8\}$ (CIFAR) and $\varepsilon\in\{0.5,1,2,4\}$ (FashionMNIST), and for CIFAR-100, are given in Appendix~\ref{sec:app_exp}.

\para{Communication accounting. }
We report accuracy with respect to two different measures. The number of rounds is the relevant measure whenever a synchronization round is the expensive unit: in cross-silo deployments, each round incurs orchestration and, under secure aggregation, protocol set-up costs. The number of total scalars communicated is a relevant measure when bandwidth is the constraint.

For a linear model on $p$-dimensional features and with $m$ classes, the model has $pm$ parameters, so each gradient descent round costs $pm$ scalars in each direction for all methods, and thus $T_g$ rounds cost $T_g\,pm$. The methods differ only in the cost of constructing the preconditioner: the naïve federated DP-FC uploads and broadcasts the full $p\times p$ covariance, with $p^2$ scalars in each direction, whereas our proposed method costs $T\,pk$ scalars in each direction for the $T$ power iterations and $k^2$ for the projected covariance $\widetilde{\mathbf Y}$. Sending only the upper triangles of the symmetric matrices would reduce these to $p(p+1)/2$ and $k(k+1)/2$; our plots count the full matrices.

Left sides of Figures~\ref{fig:cifar10_eps2_1} and~\ref{fig:fmnist_eps2_1} report accuracy as a function of communication rounds. At $\varepsilon=2$, both preconditioned methods converge faster, and to higher accuracy, than all first-order baselines, showing the benefit of second-order information. Moreover, the low-rank preconditioner obtained by our communication-efficient DP-PCA closely matches the convergence of vanilla DP-FC while avoiding the transmission of the full $p\times p$ covariance.

Right sides of Figures~\ref{fig:cifar10_eps2_1} and~\ref{fig:fmnist_eps2_1} measure performance against total communication. Here the benefit of the proposed covariance estimate becomes clear: because the one-off cost scales as $O(Tpk+k^2)$ rather than $O(p^2)$, our method reaches the final accuracy of vanilla DP-FC at a fraction of its total communication, and achieves the best accuracy--communication trade-off among all methods at this privacy level.

At the smallest budgets (FashionMNIST, $\varepsilon=0.5$; CIFAR-100, $\varepsilon=1$; Appendix~\ref{sec:app_exp}), first-order methods are competitive or better, as noise in the estimated subspace outweighs the curvature gain. In this regime, vanilla DP-FC also degrades after about 20 rounds on CIFAR-100, consistent with its $p\times p$ noise making $\widetilde{\mathbf{K}}+\lambda\mathbf{I}$ ill-conditioned; our $k\times k$ noise does not show this.

\begin{figure}[h!]
\centering
\includegraphics[width=0.49\columnwidth]{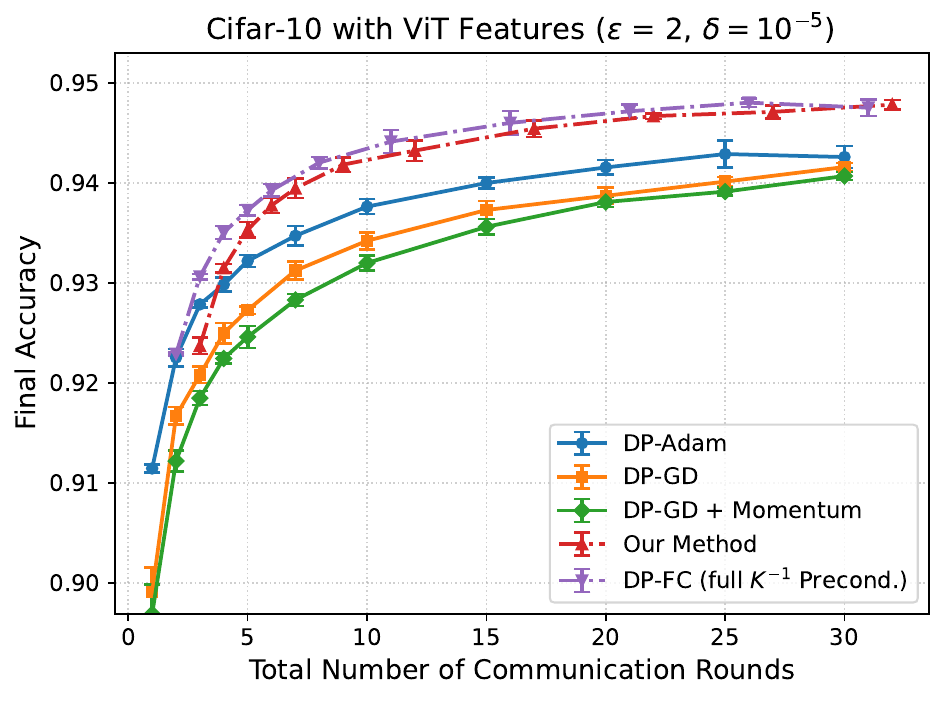}
\includegraphics[width=0.49\columnwidth]{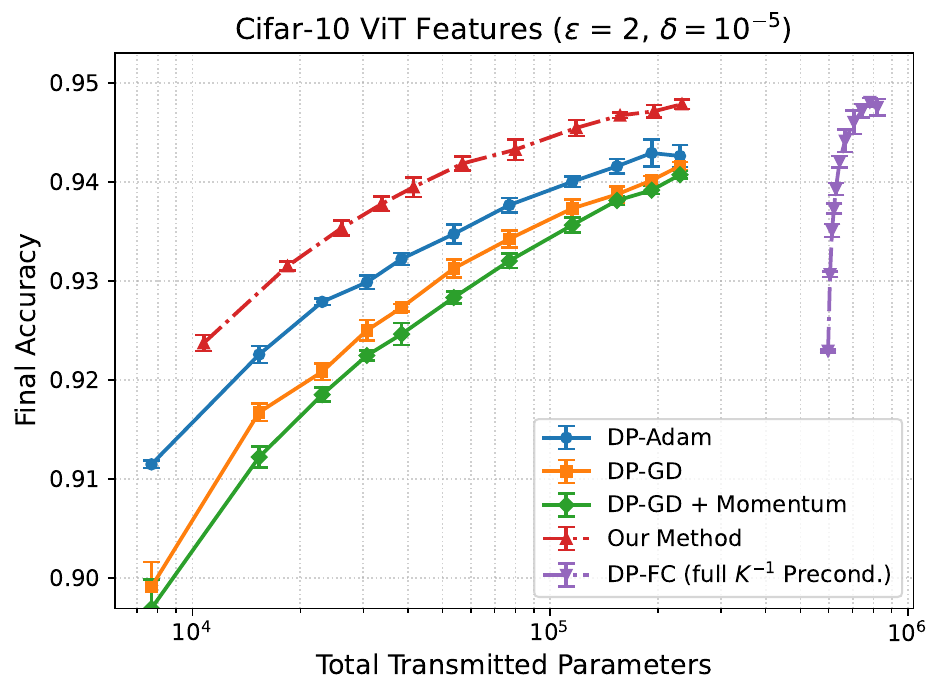}
\caption{CIFAR-10 classification with frozen ViT-B/16 features for $\varepsilon=2$ ($\delta=10^{-5}$). Left: classification accuracy as a function of the total number of communication rounds; each point corresponds to independent runs with that number of rounds. Our DP low-rank preconditioner achieves accuracy comparable to that of vanilla DP-FC. Right: classification accuracy as a function of total communication (number of transmitted scalars) for $\varepsilon=2$ ($\delta=10^{-5}$). Our method achieves the best accuracy--communication trade-off across the entire communication range.}
\label{fig:cifar10_eps2_1}
\end{figure}

\begin{figure}[h!]
\centering
\includegraphics[width=0.49\columnwidth]{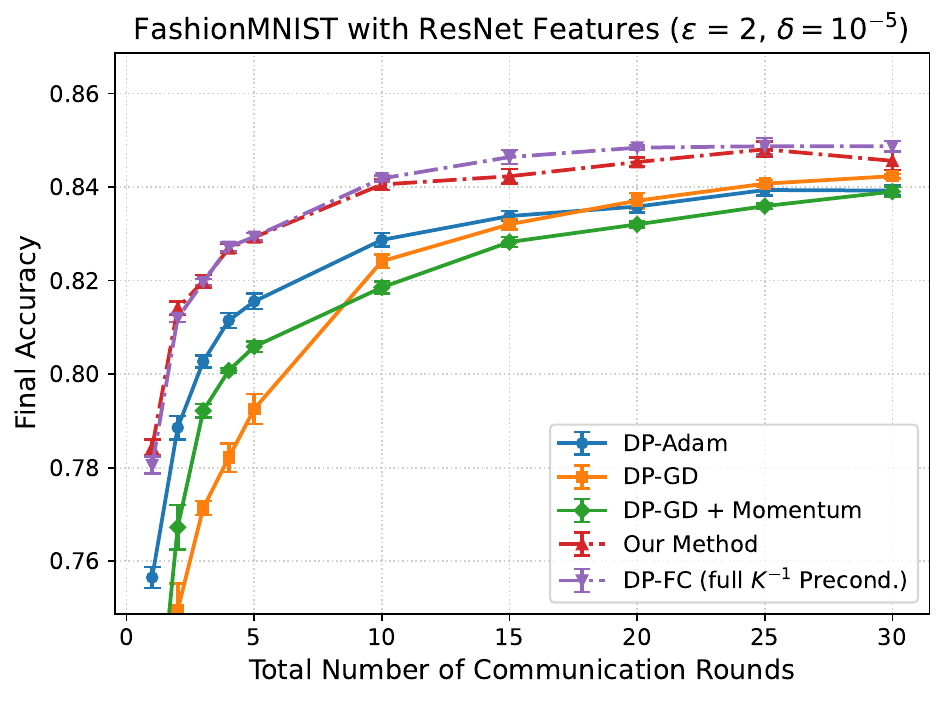}
\includegraphics[width=0.49\columnwidth]{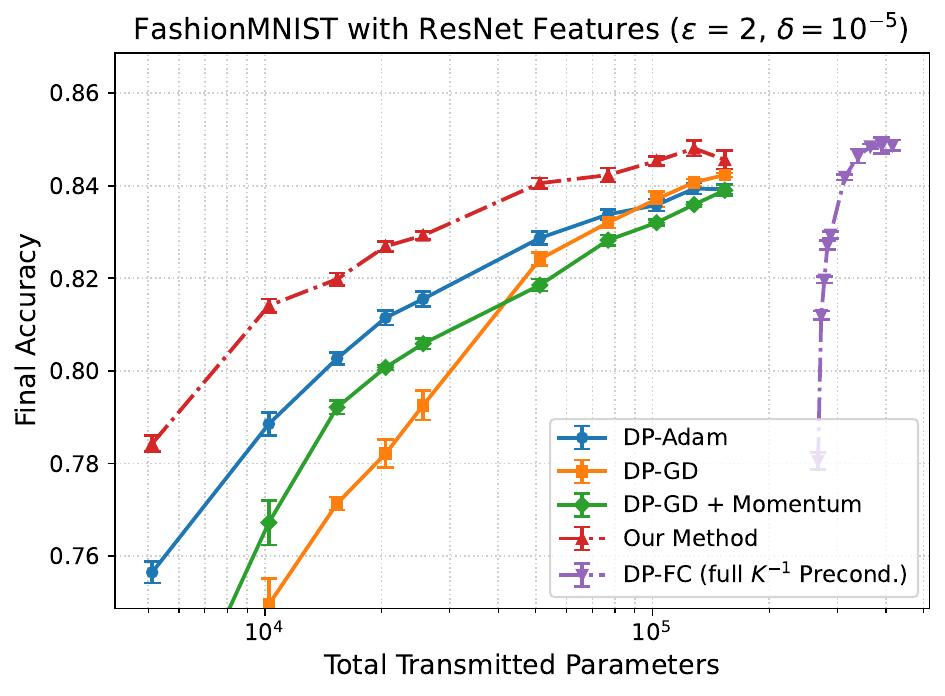}
\caption{FashionMNIST with ResNet-34 features ($p=512$), $\varepsilon=2$ ($\delta=10^{-5}$). Left: accuracy vs.\ rounds. Right: accuracy vs.\ total communication. Setup as in Figure~\ref{fig:cifar10_eps2_1}.}
\label{fig:fmnist_eps2_1}
\end{figure}

\section{Conclusions}

We presented the first federated DP low-rank covariance estimation framework that simultaneously achieves $O(pk)$ per-round communication and benefits from summation protocols, closing a gap left open by prior work.
Our three-stage procedure, DP rank selection via a shared random projection, a federated DP power method for the leading $k$ left singular vectors, and an Analyze Gauss release on the projected covariance, decomposes into per-client sums and is therefore end-to-end compatible with secure aggregation.
As a preconditioner, the resulting estimate gives an approximate Newton step, and as a consequence our proposed optimizer needs only a fraction of the rounds of DP-GD with similar communication.
Empirically, our DP-PCA method matches the centralized $O(p^2)$ Analyze Gauss baseline at a fraction of the communication cost. In federated DP optimization of GLMs, our optimization method matches vanilla DP-FC in rounds and gives the best accuracy-communication trade-off among all compared methods, except at the smallest privacy budgets.

Extending the framework from the full-batch setting to streaming or random mini-batch sampling is a natural next step, as is integrating covariance preconditioning into more general private training pipelines, e.g., for multilayer neural networks.

\bibliography{dpgradpca}
\bibliographystyle{plainnat}

\clearpage
\appendix

\section{Proof of Lemma~\ref{lem:gauss_dp}}
\label{app:proof_lemma2}

We give a proof of Lemma~\ref{lem:gauss_dp} based on the hockey-stick divergence and
on the privacy loss random variable (PRV) framework. 
\subsection{Hockey-Stick Divergence, Dominating Pairs, and PRVs}
\label{app:hsd}

For $\alpha > 0$, the hockey-stick divergence from a distribution $P$ to a distribution
$Q$ is
\begin{equation*}
    H_\alpha(P \,\|\, Q)
    \;=\; \int \big[\, P(t) - \alpha \cdot Q(t) \,\big]_+ \, \mathrm{d}t,
\end{equation*}
where $[t]_+ = \max\{0,t\}$. Tight $(\veps,\delta)$-DP guarantees can be characterized
through this divergence: for $\veps \geq 0$, the smallest $\delta$ such that
$\mathcal{M}$ is $(\veps,\delta)$-DP equals
\begin{equation} \label{eq:tight_delta}
    \delta(\veps) \;=\; \max_{D \simeq D'} \,
    H_{\ee^{\veps}}\!\big( \mathcal{M}(D) \,\|\, \mathcal{M}(D') \big),
\end{equation}
where the maximum is over neighboring datasets~\citep{zhu2022optimal}.

For some mechanisms, such as the Gaussian mechanism or the Poisson subsampled Gaussian mechanism, we can bound $H_{\ee^\veps}$ uniformly by the divergence
of a single, dataset-independent pair of distributions, a so-called dominating pair of distributions.

\begin{defn}[\citealt{zhu2022optimal}] \label{def:dominating}
A pair of distributions $(P,Q)$ dominates $\mathcal{M}$ if, for all neighboring
datasets $D \simeq D'$ and all $\alpha > 0$,
\begin{equation*}
    H_\alpha\!\big( \mathcal{M}(D) \,\|\, \mathcal{M}(D') \big)
    \;\leq\; H_\alpha(P \,\|\, Q).
\end{equation*}
\end{defn}

Dominating pairs compose multiplicatively under adaptive composition.

\begin{thm}[\citealt{zhu2022optimal}] \label{thm:composition}
If $(P_i, Q_i)$ dominates $\mathcal{M}_i$ for $i = 1, \ldots, T$, then
$\big( \prod_i P_i,\, \prod_i Q_i \big)$ dominates the adaptive composition
$\mathcal{M}_T \circ \cdots \circ \mathcal{M}_1$.
\end{thm}

The divergence-based analysis of compositions can be made  tractable via privacy loss random variables.

\begin{defn} \label{def:prv}
Let $P$ and $Q$ be probability density functions with $P \ll Q$. The
{privacy loss random variable} (PRV) $\omega_{P/Q}$ associated with $(P,Q)$ is
\begin{equation*}
    \omega_{P/Q} \;=\; \log \frac{P(t)}{Q(t)}, \qquad t \sim P.
\end{equation*}
\end{defn}

\begin{thm}[\citealt{gopi2021numerical}] \label{thm:gopi}
For any $\veps \geq 0$,
\begin{equation} \label{eq:hsd_prv}
    H_{\ee^\veps}(P \,\|\, Q)
    \;=\; \mathop{\mathbb{E}}_{s \sim \omega_{P/Q}}
          \!\left[\, 1 - \ee^{\veps - s} \,\right]_+.
\end{equation}
Moreover, the PRV of a product pair $(P_1 \times \cdots \times P_T,\,
Q_1 \times \cdots \times Q_T)$ equals the sum of the per-coordinate PRVs:
\begin{equation*}
    \omega_{\prod_i P_i \,/\, \prod_i Q_i}
    \;\stackrel{d}{=}\;
    \sum_{i=1}^T \omega_{P_i / Q_i},
\end{equation*}
where the summands on the right-hand side are independent.
\end{thm}

Combining \eqref{eq:tight_delta}, Theorem~\ref{thm:composition} and
Theorem~\ref{thm:gopi}, in order to obtain $(\veps,\delta)$-DP bounds for an adaptive
composition it suffices to (i) exhibit a dominating pair for each step, (ii) sum the
corresponding PRVs, and (iii) evaluate the expectation in~\eqref{eq:hsd_prv}.

\subsection{The PRV of a Gaussian mechanism}
\label{app:gaussian_prv}

Consider a Gaussian mechanism $\mathcal{M}(D) = \mathbf{f}(D) + \mathbf{Z}$ with
$\mathbf{Z} \sim \mathcal{N}(0, \sigma^2 I_d)$ and $L_2$-sensitivity $\Delta$. A
tightly dominating pair is given by the one-dimensional Gaussians
\begin{equation} \label{eq:gauss_dom_pair}
    P \;=\; \mathcal{N}(0, \sigma^2), \qquad
    Q \;=\; \mathcal{N}(\Delta, \sigma^2),
\end{equation}
and a direct computation~\citep[see, e.g.,][Lemma~11]{sommer2019privacy} shows that
the corresponding PRV is itself Gaussian:
\begin{equation} \label{eq:prv_gauss}
    \omega_{P/Q}
    \;\sim\; \mathcal{N}\!\left( \frac{\Delta^2}{2\sigma^2},\,
                                  \frac{\Delta^2}{\sigma^2} \right).
\end{equation}

Plugging~\eqref{eq:prv_gauss} into~\eqref{eq:hsd_prv} yields the analytical Gaussian
bound of~\citet{balle2018improving}: for every $\veps \geq 0$,
\begin{equation} \label{eq:balle_wang}
    H_{\ee^\veps}(P \,\|\, Q)
    \;=\; \Phi\!\left( -\frac{\veps \sigma}{\Delta} + \frac{\Delta}{2\sigma} \right)
        - \ee^{\veps}\,
          \Phi\!\left( -\frac{\veps \sigma}{\Delta} - \frac{\Delta}{2\sigma} \right),
\end{equation}
where $\Phi$ denotes the cumulative distribution function of the standard univariate
Gaussian. 

\subsection{Proof of Lemma~\ref{lem:gauss_dp}}
\label{app:proof}

We are now ready to prove the lemma.

\begin{proof}[Proof of Lemma~\ref{lem:gauss_dp}]
For each $i \in [T]$, let $\mathcal{M}_i(D) = \mathbf{f}_i(D) + \mathbf{Z}_i$ with
$\mathbf{Z}_i \sim \mathcal{N}(0, \sigma_i^2 I_{d_i})$ and $L_2$-sensitivity
$\Delta_i$. By~\eqref{eq:gauss_dom_pair}, the one-dimensional Gaussians
\begin{equation*}
    P_i \;=\; \mathcal{N}(0, \sigma_i^2), \qquad
    Q_i \;=\; \mathcal{N}(\Delta_i, \sigma_i^2)
\end{equation*}
form a (tightly) dominating pair for $\mathcal{M}_i$, and~\eqref{eq:prv_gauss} gives
\begin{equation*}
    \omega_{P_i / Q_i}
    \;\sim\; \mathcal{N}\!\left( \frac{\Delta_i^2}{2 \sigma_i^2},\,
                                  \frac{\Delta_i^2}{\sigma_i^2} \right).
\end{equation*}

By Theorem~\ref{thm:composition}, $\big( \prod_i P_i,\, \prod_i Q_i \big)$ dominates
the adaptive composition $\mathcal{M}^{(T)} = \mathcal{M}_T \circ \cdots \circ
\mathcal{M}_1$. By Theorem~\ref{thm:gopi}, the PRV of this product pair is the sum of
independent Gaussians, hence Gaussian:
\begin{equation*}
    \omega \;:=\; \sum_{i=1}^{T} \omega_{P_i / Q_i}
    \;\sim\; \mathcal{N}\!\left(
        \sum_{i=1}^{T} \frac{\Delta_i^2}{2 \sigma_i^2},\;
        \sum_{i=1}^{T} \frac{\Delta_i^2}{\sigma_i^2}
    \right)
    \;=\; \mathcal{N}\!\left( \frac{\widetilde{\Delta}^2}{2},\,
                              \widetilde{\Delta}^2 \right),
\end{equation*}
where we used the definition $\widetilde{\Delta}^2 = \sum_{i=1}^{T} \Delta_i^2 /
\sigma_i^2$. The analytical expression for the privacy profile determined by the PRV $\mathcal{N}\!\left( \frac{\widetilde{\Delta}^2}{2},\, \widetilde{\Delta}^2 \right)$ is given in~\citep{balle2018improving} from which the claim follows.

\end{proof}

The following closed-form calibration, used in Theorems~\ref{thm:error_two_phases} and~\ref{thm:ref_bound}, is a (conservative) consequence of Lemma~\ref{lem:gauss_dp}.

\begin{cor} \label{cor:calibration}
Consider an adaptive composition of $T$ Gaussian mechanisms, each with $L_2$-sensitivity $\Delta$ and noise standard deviation $\Delta \sigma$. Let $\delta \in (0,1)$, $0 < \veps \leq 2\log(1/\delta)$, and $\sigma = \veps^{-1}\sqrt{4T\log(1/\delta)}$. Then the composition is $(\veps,\delta)$-DP.
\end{cor}
\begin{proof}
By~\citep[Prop.~1.6 and Lemma~1.7]{Bun2016}, the composition is $\rho$-zCDP with $\rho = T/(2\sigma^2) = \veps^2/(8L)$, where $L = \log(1/\delta)$. By~\citep[Prop.~1.3]{Bun2016}, it is then $(\veps',\delta)$-DP with
$$
\veps' = \rho + 2\sqrt{\rho L} = \frac{\veps^2}{8L} + \frac{\veps}{\sqrt{2}} \leq \veps,
$$
where the inequality holds for $\veps \leq 8(1 - 1/\sqrt{2})L$, and in particular for $\veps \leq 2L$.
\end{proof}

\section{Proof for Section~\ref{sec:initialization}}

\subsection{Proof of Lemma~\ref{lem:dp_initialization}}
\begin{proof}
Under the add/remove relation, let $\mathbf{y}\in\mathbb{R}^r$ denote the clipped projection of the added or removed column, so that $\|\mathbf{y}\|_2\le\bar C$. The released matrix $\wt{\mathbf{X}}\wt{\mathbf{X}}^\top$ changes by $\mathbf{y}\mathbf{y}^\top$, and
$\|\mathbf{y}\mathbf{y}^\top\|_F=\|\mathbf{y}\|_2^2\le \bar C^2$.
The upper-triangular vectorisation has $\ell_2$-norm at most the Frobenius norm.
\end{proof}

\subsection{Proof of Lemma~\ref{lem:rank_no_fp}}
\begin{proof}
By Lemma~\ref{thm:symmetric_concentration} with $t = 2\sqrt{\log(2/\rho)}$, $\|\mathbf{E}\|_2\le b_r$ with probability at least $1-\rho$. On this event, Weyl's inequality gives
$|\lambda_i(\wt{\mathbf{C}})-\lambda_i(\wt{\mathbf{X}}\wt{\mathbf{X}}^\top)|\le\sigma\,b_r$ for all $i$. Hence $\lambda_i(\wt{\mathbf{X}}\wt{\mathbf{X}}^\top)\ge\lambda_i(\wt{\mathbf{C}})-\sigma\,b_r>0$ for all $i\le\wh k$, which gives the upper bound, as $\operatorname{rank}(\wt{\mathbf{X}})\le\operatorname{rank}(\mathbf{X})$; and $\lambda_i(\wt{\mathbf{X}}\wt{\mathbf{X}}^\top) > 2\sigma b_r$ implies $\lambda_i(\wt{\mathbf{C}}) > \sigma b_r$, which gives the lower bound.
\end{proof}

\section{Proof of Theorem~\ref{thm:error_two_phases}}

\subsection{Auxiliary Results} 

We use the following convergence result for the noisy power method by~\citet{hardt2014noisy}~\citep[see also][]{balcan2016improved}, which implies the convergence bound of Thm.~\ref{thm:ref_bound}. Throughout this section, $\lambda_i := \lambda_i(\mathbf{X}\mathbf{X}^\top)$.

\begin{thm}[\citealt{hardt2014noisy}, Cor.~1.1]~\label{thm:gap_thm}
Let $k \leq b \leq p$. Let $\mathbf{U}_k \in \mathbb{R}^{p \times k}$ be the top-$k$ eigenvectors of a positive semi-definite matrix $\mathbf{K} \in \mathbb{R}^{p \times p}$ with eigenvalues $\lambda_1 \geq \cdots \geq \lambda_p \geq 0$. Suppose the noisy power method is run with block size $b$ (iterates $\mathbf{V}_\ell \in \mathbb{R}^{p \times b}$), starting from an orthonormal basis of a uniformly random $b$-dimensional subspace, and that at every iteration the noise matrix $\mathbf{G}_\ell \in \mathbb{R}^{p \times b}$ satisfies
\begin{equation*}
        5\|\mathbf{G}_\ell\|_2 \leq \zeta(\lambda_k - \lambda_{k+1}), \qquad
5\|\mathbf{U}_k^\top \mathbf{G}_\ell\|_2 \leq (\lambda_k - \lambda_{k+1}) \frac{\sqrt{b} - \sqrt{k-1}}{\tau \sqrt{p}}
\end{equation*}
for some $\tau > 0$ and $\zeta < 1/2$.
Then, with probability at least $1 - \tau^{-\Omega(b+1-k)} - e^{-\Omega(p)}$, there exists
$
L = O\left( \frac{\lambda_k}{\lambda_k - \lambda_{k+1}} \log\left( \frac{\tau p}{\zeta} \right) \right)
$
such that
$
\norm{(\mathbf{I} - \mathbf{V}_T \mathbf{V}_T^\top)\mathbf{U}_k}_2 \leq \zeta
$
for all $T \geq L$.
\end{thm}

To analyze the effect of the additive noise of Analyze Gauss, we will need a concentration bound for the spectral norm of symmetric Gaussian matrices. 

\begin{lem}[Thm.\;II.7, \citealp{davidson2001local}] \label{thm:symmetric_concentration}
Let $\mathbf{E} \in \mathbb{R}^{r \times r}$ be a symmetric Gaussian matrix where entries $E_{ij}$ ($i \leq j$) are i.i.d. $\mathcal{N}(0,1)$ and let $t \geq 0$. Then, the spectral norm satisfies
$$
\mathbb{P}\left( \|\mathbf{E}\|_2 \geq 2\sqrt{r} + t \right) \leq 2 \exp\left( - \frac{t^2}{4} \right).
$$
\end{lem}

We will also need the following concentration bound for converting the bounds of Thm.\;\ref{thm:gap_thm}.

\begin{lem}[\citealp{rudelson2010non}, Eq.~(2.3)]\label{thm:rv_gaussian}
Let $\mathbf{A} \in \mathbb{R}^{N \times n}$ be a matrix with i.i.d.\ $\mathcal{N}(0,1)$ entries. Then, for every $t \geq 0$,
$$
\mathbb{P}\!\left( s_{\max}(\mathbf{A}) \;\leq\; \sqrt{N} + \sqrt{n} + t \right) \;\geq\; 1 - 2 e^{-t^2/2}.
$$
\end{lem}

The following result gives a bound for the projection error $\norm{ (\mathbf{I} - \mathbf{V}_i \mathbf{V}_i^\top) \mathbf{U}_k }_2$, where  $\mathbf{U}_k \in \mathbb{R}^{p \times k}$ gives the principal $k$ eigenvectors. It is an adaptation of~\citep[Thm 1.3,][]{hardt2014noisy} to our noise calibration.
We will need this result in the overall error bound given in Thm.~\ref{thm:error_two_phases}.

\begin{thm} \label{thm:ref_bound}
Let $\mathbf{U}_k \in \mathbb{R}^{p \times k}$ be the matrix of leading $k$ eigenvectors of $\mathbf{X}\mathbf{X}^\top$, and let $\mathbf{V}_T$ be the output of the noisy power method (Algorithm~\ref{alg:DP_power}) with block size $k$, a public number of iterations $T$ and i.i.d.\ noise entries $(\mathbf{Z}_t)_{ij} \sim \mathcal{N}(0, C^4 \sigma^2)$, where $C$ is the public bound on $\|\mathbf{x}_i\|_2$.
\begin{enumerate}
\item[(i)] The algorithm is $\sqrt{T}/\sigma$-GDP. In particular, if $\sigma = \varepsilon^{-1}\sqrt{4 \, T \log(1/\delta)}$ for some $\delta \in (0,1)$ and $0 < \varepsilon \leq 2\log(1/\delta)$, the algorithm satisfies $(\varepsilon,\delta)$-DP.
\item[(ii)] Fix a constant $\tau \geq 1$, let $\nu \in (0,1)$, and let
$$
\bar\zeta := \frac{10\,\tau\, C^2\,\sigma\sqrt{p}\,\bigl(2 + \sqrt{2\log(4T/\nu)}\bigr)}
     {\lambda_k - \lambda_{k+1}}
\cdot \frac{\sqrt{k}}{\sqrt{k} - \sqrt{k-1}}
\;=\; O\!\left(\frac{C^2\,\sigma\sqrt{p}\,\bigl(1 + \sqrt{\log(T/\nu)}\bigr)}
     {\lambda_k - \lambda_{k+1}}
\cdot k \right).
$$
If $T \geq L$ for $L = O\!\left(\tfrac{\lambda_k}{\lambda_k - \lambda_{k+1}}\log(\tau p / \bar\zeta) \right)$ as in Theorem~\ref{thm:gap_thm}, then, w.p. at least $1 - \nu - p_0$ with $p_0 = \tau^{-\Omega(1)} + e^{-\Omega(p)}$,
\begin{equation*}
\bigl\| (\mathbf{I} - \mathbf{V}_T \mathbf{V}_T^\top)\mathbf{U}_k \bigr\|_2
\;\leq\; \min\{1, \bar\zeta\}.
\end{equation*}
\end{enumerate}
\begin{proof}
(i) Each iteration is a Gaussian mechanism with sensitivity $C^2$ and noise standard deviation $C^2\sigma$ (Section~\ref{sec:comm_eff_pca}), so Lemma~\ref{lem:gauss_dp} gives $\widetilde{\Delta} = \sqrt{T}/\sigma$, and the $(\varepsilon,\delta)$-claim follows from Corollary~\ref{cor:calibration} with $\Delta = C^2$.

(ii) We apply Theorem~\ref{thm:gap_thm} with block size $b = k$, so that $\tau^{-\Omega(b+1-k)} = \tau^{-\Omega(1)}$ and the two conditions read
\begin{equation}\label{eq:hp-conditions}
5\|\mathbf{G}_t\|_2 \;\le\; \zeta\,(\lambda_k - \lambda_{k+1}),
\qquad
5\|\mathbf{U}_k^\top \mathbf{G}_t\|_2 \;\le\; (\lambda_k - \lambda_{k+1})
\,\frac{\sqrt{k}-\sqrt{k-1}}{\tau\sqrt{p}}.
\end{equation}

At each iteration, $\mathbf{G}_t$ is a $p \times k$ matrix with i.i.d. $\mathcal{N}(0, C^4 \sigma^2)$ entries. By the (constant-free) Gaussian concentration bound for the largest singular value~\citep[Eq.\,(2.3)]{rudelson2010non}, applied after factoring out the scale $C^2\sigma$, for any $s \ge 0$,
$$
\mathbb{P}\!\left(\|\mathbf{G}_t\|_2
\;\le\; C^2\sigma\bigl(\sqrt{p} + \sqrt{k} + s\bigr)\right)
\;\ge\; 1 - 2e^{-s^2/2}.
$$
Since $\mathbf{U}_k^\top \mathbf{G}_t \in \mathbb{R}^{k \times k}$ is again a Gaussian matrix with i.i.d. $\mathcal{N}(0, C^4 \sigma^2)$ entries (by rotational invariance of the Gaussian and the orthonormality of $\mathbf{U}_k$), the same bound gives
$$
\mathbb{P}\!\left(\|\mathbf{U}_k^\top \mathbf{G}_t\|_2
\;\le\; C^2\sigma\bigl(2\sqrt{k} + s\bigr)\right)
\;\ge\; 1 - 2e^{-s^2/2}.
$$
Setting $s = \sqrt{2\log(4T/\nu)}$ and taking a union bound over the $T$
iterations and the two events at each iteration, we obtain that
simultaneously for all $t \le T$, with probability at least $1 - \nu$,
\begin{equation}\label{eq:gauss-norms}
\|\mathbf{G}_t\|_2 \;\le\; C^2\sigma\sqrt{p}\,(2 + s),
\qquad
\|\mathbf{U}_k^\top \mathbf{G}_t\|_2 \;\le\; C^2\sigma\sqrt{k}\,(2 + s),
\end{equation}
where we used $\sqrt{p} + \sqrt{k} + s \leq \sqrt{p}(2+s)$ and $2\sqrt{k} + s \leq \sqrt{k}(2+s)$ (since $1 \leq k \le p$). We work on this event and on the success event of the initialization in Theorem~\ref{thm:gap_thm}; a union bound gives the stated probability.
Let
$$
\zeta' := \frac{5C^2\sigma\sqrt{p}\,(2+s)}{\lambda_k - \lambda_{k+1}},
\qquad \text{so that} \qquad \bar\zeta = 2\tau\,\zeta'\, \frac{\sqrt{k}}{\sqrt{k}-\sqrt{k-1}} \geq 2\zeta'.
$$
By~\eqref{eq:gauss-norms}, the first condition of~\eqref{eq:hp-conditions} holds with $\zeta = \zeta'$. Since $\|(\mathbf{I} - \mathbf{V}_T\mathbf{V}_T^\top)\mathbf{U}_k\|_2 \leq 1$ always holds, it suffices to consider the case $\bar\zeta < 1$, which implies $\zeta' < 1/2$. In this case also the second condition of~\eqref{eq:hp-conditions} holds: otherwise, by~\eqref{eq:gauss-norms},
$$
(\lambda_k - \lambda_{k+1})\frac{\sqrt{k}-\sqrt{k-1}}{\tau\sqrt{p}} < 5\|\mathbf{U}_k^\top \mathbf{G}_t\|_2 \leq 5C^2\sigma\sqrt{k}\,(2+s),
$$
i.e., $\tau \zeta' \sqrt{k}/(\sqrt{k}-\sqrt{k-1}) > 1$ and hence $\bar\zeta > 2$, a contradiction. Theorem~\ref{thm:gap_thm} with $\zeta = \zeta'$ then gives $\|(\mathbf{I} - \mathbf{V}_T\mathbf{V}_T^\top)\mathbf{U}_k\|_2 \leq \zeta' \leq \bar\zeta$ for all $T \geq L$, where $L = O\big(\tfrac{\lambda_k}{\lambda_k - \lambda_{k+1}}\log\tfrac{\tau p}{\zeta'}\big) = O\big(\tfrac{\lambda_k}{\lambda_k - \lambda_{k+1}}\log\tfrac{\tau p}{\bar\zeta}\big)$ since $\zeta' \geq \bar\zeta/(4\tau k)$ and $k \leq p$. Finally, $\frac{\sqrt{k}}{\sqrt{k}-\sqrt{k-1}} = \sqrt{k}(\sqrt{k}+\sqrt{k-1}) \leq 2k$ gives the $O(\cdot)$-form of $\bar\zeta$.
\end{proof}
\end{thm}

Let $\mathbf{K} = \mathbf{X}\mathbf{X}^\top$ and let $\mathbf{Q} \in \mathbb{R}^{p \times k}$ have orthonormal columns with
$
\norm{ (\mathbf{I} - \mathbf{Q} \mathbf{Q}^\top) \mathbf{U}_k }_2 \leq \zeta
$
(as given by Theorem~\ref{thm:ref_bound}).
This directly translates to the projection error~\citep[cf.][Eq.~(3)]{hardt2014noisy}
\begin{equation} \label{eq:proj_error_A}
\norm{\mathbf{K} - \mathbf{Q} \mathbf{Q}^\top \mathbf{K}}_2 \leq \lambda_{k+1} + \zeta \lambda_1,
\end{equation}
since, writing $\mathbf{K} = \mathbf{U}_k \boldsymbol{\Lambda}_k \mathbf{U}_k^\top + \mathbf{U}_\perp \boldsymbol{\Lambda}_\perp \mathbf{U}_\perp^\top$, the term $(\mathbf{I} - \mathbf{Q} \mathbf{Q}^\top)\mathbf{U}_k \boldsymbol{\Lambda}_k \mathbf{U}_k^\top$ has spectral norm at most $\zeta\lambda_1$ and $(\mathbf{I} - \mathbf{Q} \mathbf{Q}^\top)\mathbf{U}_\perp \boldsymbol{\Lambda}_\perp \mathbf{U}_\perp^\top$ at most $\lambda_{k+1}$.

This can be further translated to a bound for the error $\norm{\mathbf{K} - \mathbf{Q} \mathbf{Q}^\top \mathbf{K} \mathbf{Q} \mathbf{Q}^\top}_2$ using the following auxiliary lemma.

\begin{lem} \label{lem:triangle_PAP}
Let $\mathbf{A} \in \mathbb{R}^{m \times m}$ be symmetric and let $\mathbf{Q} \in \mathbb{R}^{m \times r}$ have orthonormal columns, i.e., $\mathbf{Q}^\top \mathbf{Q} = \mathbf{I}_r$. Then the following inequality holds:
$$
\norm{\mathbf{A} - \mathbf{Q}\mathbf{Q}^\top \mathbf{A} \mathbf{Q}\mathbf{Q}^\top}_2 \leq \sqrt{2} \, \norm{\mathbf{A} - \mathbf{Q}\mathbf{Q}^\top \mathbf{A}}_2.
$$
\begin{proof}
Define the orthogonal projector $\mathbf{P} := \mathbf{Q}\mathbf{Q}^\top$. Then
$$
\mathbf{A} - \mathbf{P} \mathbf{A} \mathbf{P} = (\mathbf{I} - \mathbf{P}) \mathbf{A} + \mathbf{P} \mathbf{A} (\mathbf{I} - \mathbf{P}).
$$
The two terms on the right-hand side have orthogonal column spaces, so
$$
\norm{\mathbf{A} - \mathbf{P} \mathbf{A} \mathbf{P}}_2 \leq \sqrt{ \norm{(\mathbf{I} - \mathbf{P}) \mathbf{A}}_2^2 + \norm{\mathbf{P} \mathbf{A} (\mathbf{I} - \mathbf{P})}_2^2 }.
$$
By symmetry of $\mathbf{A}$, $\norm{\mathbf{P} \mathbf{A} (\mathbf{I} - \mathbf{P})}_2 = \norm{(\mathbf{I} - \mathbf{P}) \mathbf{A} \mathbf{P}}_2 \leq \norm{(\mathbf{I} - \mathbf{P}) \mathbf{A}}_2$, so both terms are bounded by $\norm{\mathbf{A} - \mathbf{P} \mathbf{A}}_2$, and the result follows. 
\end{proof}
\end{lem}
With the bound~\eqref{eq:proj_error_A} and with Lemma~\ref{lem:triangle_PAP}, the error analysis for the two phases (computing the basis and the noisy projected matrix) becomes straightforward. 

\subsection{The Proof of Theorem~\ref{thm:error_two_phases}}

\begin{proof}
\emph{(i)} By Theorem~\ref{thm:ref_bound}(i) with $\sigma = \sigma_1/C^2$, the first stage (whose output $ \mathbf{V}_k$ is a post-processing of the noisy iterates) is $(\varepsilon_1,\delta_1)$-DP Given $\mathbf{V}_k$, the second stage is a single Gaussian mechanism with sensitivity $C^2$ (Section~\ref{sec:comm_eff_pca}) and noise standard deviation $\sigma_2$, hence $(\varepsilon_2,\delta_2)$-DP by Corollary~\ref{cor:calibration} with $T=1$ and $\Delta = C^2$. The claim follows from basic (adaptive) composition.

\emph{(ii)} Substituting $\widetilde{\mathbf{Y}}_k = \mathbf{V}_k^\top \mathbf{K}\mathbf{V}_k + \mathbf{Z}$,
the error decomposes as
\begin{equation*}
\mathrm{Err}_k \;=\; \mathbf{K} - \mathbf{V}_k\widetilde{\mathbf{Y}}_k\mathbf{V}_k^\top
\;=\; \bigl(\mathbf{K} - \mathbf{V}_k\mathbf{V}_k^\top \mathbf{K}\mathbf{V}_k\mathbf{V}_k^\top\bigr)
   \;-\; \mathbf{V}_k\mathbf{Z}\mathbf{V}_k^\top.
\end{equation*}
Since $\mathbf{V}_k$ has orthonormal columns, $\|\mathbf{V}_k\mathbf{Z}\mathbf{V}_k^\top\|_2 = \|\mathbf{Z}\|_2$,
and the triangle inequality gives
\begin{equation}\label{eq:two-term-decomposition}
\|\mathrm{Err}_k\|_2 \;\le\;
\|\mathbf{K} - \mathbf{V}_k\mathbf{V}_k^\top\mathbf{K}\mathbf{V}_k\mathbf{V}_k^\top\|_2
\;+\;
\|\mathbf{Z}\|_2.
\end{equation}
We bound the two terms separately.

By Lemma~\ref{lem:triangle_PAP} ($\mathbf{K}$ is symmetric),
$$
\|\mathbf{K} - \mathbf{V}_k\mathbf{V}_k^\top \mathbf{K}\mathbf{V}_k\mathbf{V}_k^\top\|_2
\;\le\; \sqrt{2}\,\|\mathbf{K} - \mathbf{V}_k\mathbf{V}_k^\top \mathbf{K}\|_2.
$$
Combining this with~\eqref{eq:proj_error_A} and Theorem~\ref{thm:ref_bound}(ii) (with $C^2\sigma = \sigma_1$), we obtain, with probability at least $1-\nu-p_0$,
\begin{equation}\label{eq:term-I}
\|\mathbf{K} - \mathbf{V}_k\mathbf{V}_k^\top \mathbf{K} \mathbf{V}_k\mathbf{V}_k^\top\|_2 \;\le\;
\sqrt{2}\bigl(\lambda_{k+1} + \lambda_1 \zeta\bigr), \qquad \zeta = \min\{1,\bar\zeta\},
\end{equation}
with $\bar\zeta$ as in Theorem~\ref{thm:ref_bound}.

Write $\mathbf{Z} = \sigma_2 \mathbf{G}$, where $\mathbf{G} \in \mathbb{R}^{k\times k}$ is a symmetric Gaussian matrix whose upper-triangular entries (including the diagonal) are i.i.d.\ $\mathcal{N}(0,1)$. By Lemma~\ref{thm:symmetric_concentration}, for any $t \ge 0$,
$$
\mathbb{P}\bigl(\|\mathbf{G}\|_2 \;\le\; 2\sqrt{k} + t\bigr) \;\ge\; 1 - 2e^{-t^2/4}.
$$
Setting $t = 2\sqrt{\log(2/\rho)}$, 
we obtain, with probability at least $1-\rho$,
\begin{equation}\label{eq:term-II}
\|\mathbf{Z}\|_2 \;\le\; \sigma_2\Bigl(2\sqrt{k} + 2\sqrt{\log(2/\rho)}\Bigr).
\end{equation}

Substituting~\eqref{eq:term-I} and~\eqref{eq:term-II} into~\eqref{eq:two-term-decomposition} and applying a union bound over the high-probability events yields the stated inequality with probability at least $1 - \nu - \rho - p_0$.
\end{proof}

\section{Proofs for Section~\ref{sec:convergence}}
\label{app:convergence_proofs}

Write $e_j:=\theta_j-\theta_*$. Since $\nabla\mathcal{L}(\theta)=\mathbf{C}_\lambda(\theta-\theta_*)$,
\eqref{eq:pdpgd} reads
$e_{j+1}=(\mathbf{I}-\eta\mathbf{P}^{-1}\mathbf{C}_\lambda)e_j-\eta\mathbf{P}^{-1}\mathcal{E}_j$.
For $\mathbf{G}\succ0$, Courant--Fischer applied to
$\mathbf{G}^{-1/2}\mathbf{C}_\lambda\mathbf{G}^{-1/2}$ \citep[see, e.g.,][Thm.~4.2.2 and Thm.~4.2.6]{horn2012matrix} shows that the extreme generalized
eigenvalues of $(\mathbf{C}_\lambda,\mathbf{G})$ are the extreme values of
$\mathbf{x}^\top\mathbf{C}_\lambda\mathbf{x}/\mathbf{x}^\top\mathbf{G}\mathbf{x}$ over
$\mathbf{x}\neq\mathbf{0}$.

\subsection{Auxiliary Lemma}

\begin{lem}\label{lem:woodbury}
If $\mathbf{V}\in\mathbb{R}^{p\times k}$ is orthonormal and
$\wt{\mathbf{Y}}+\lambda\mathbf{I}_k\succ0$, then $\wh{\mathbf{G}}\succ0$ and
\begin{equation}\label{eq:Ghatinva}
    \wh{\mathbf{G}}^{-1}
    =\mathbf{V}\bigl(\wt{\mathbf{Y}}+\lambda\mathbf{I}_k\bigr)^{-1}\mathbf{V}^\top
    +\tfrac1\lambda\bigl(\mathbf{I}-\mathbf{V}\mathbf{V}^\top\bigr).
\end{equation}
For the oracle preconditioner
$\mathbf{G}_\star:=\mathbf{U}_k\boldsymbol{\Lambda}_k\mathbf{U}_k^\top+\lambda\mathbf{I}$
($\mathbf{V}=\mathbf{U}_k$, $\mathbf{Z}=\mathbf{0}$), the generalized eigenvalues
of $(\mathbf{C}_\lambda,\mathbf{G}_\star)$ are $1$ for $i\le k$ and $(\lambda_i+\lambda)/\lambda$
for $i>k$. They all lie in $[1,\kappa_k]$, and
$\tr(\mathbf{G}_\star^{-1})=\sum_{i\le k}(\lambda_i+\lambda)^{-1}+(p-k)/\lambda$.
\end{lem}
\begin{proof}
Let $\mathbf{O}=[\mathbf{V},\mathbf{V}_\perp]$ be orthogonal. Then
$\mathbf{O}^\top\wh{\mathbf{G}}\mathbf{O}
=\mathrm{diag}\bigl(\wt{\mathbf{Y}}+\lambda\mathbf{I}_k,\,\lambda\mathbf{I}_{p-k}\bigr)\succ0$.
Inverting blockwise and using
$\mathbf{V}_\perp\mathbf{V}_\perp^\top=\mathbf{I}-\mathbf{V}\mathbf{V}^\top$ gives
\eqref{eq:Ghatinva}. For $\mathbf{V}=\mathbf{U}_k$, $\mathbf{Z}=\mathbf{0}$ we have
$\wt{\mathbf{Y}}=\boldsymbol{\Lambda}_k$, so in the basis $\mathbf{U}$ both
$\mathbf{C}_\lambda$ and $\mathbf{G}_\star$ are diagonal. The entries of $\mathbf{C}_\lambda$ are
$\lambda_i+\lambda$. The entries of $\mathbf{G}_\star$ are $\lambda_i+\lambda$ for
$i\le k$ and $\lambda$ for $i>k$. The generalized eigenvalues are the ratios:
$1$ for $i\le k$, and $1+\lambda_i/\lambda\in[1,\kappa_k]$ for $i>k$, since
$0\le\lambda_i\le\lambda_{k+1}$. The trace is the sum of the reciprocal diagonal
entries of $\mathbf{G}_\star$.
\end{proof}

\subsection{Proof of Lemma~\ref{lem:perturb}}
\begin{proof}
Let $\mathbf{M}:=\mathbf{G}_\star^{-1/2}\wh{\mathbf{G}}\mathbf{G}_\star^{-1/2}$.
By definition of $\gamma$,
$(1-\gamma)\mathbf{I}\preccurlyeq\mathbf{M}\preccurlyeq(1+\gamma)\mathbf{I}$, and
congruence with $\mathbf{G}_\star^{1/2}$ gives the sandwich. If $\gamma<1$, then
$\wh{\mathbf{G}}\succ0$, and for $\mathbf{x}\neq\mathbf{0}$,
\begin{equation*}
\frac{\mathbf{x}^\top\mathbf{C}_\lambda\mathbf{x}}{\mathbf{x}^\top\wh{\mathbf{G}}\mathbf{x}}
=\frac{\mathbf{x}^\top\mathbf{C}_\lambda\mathbf{x}}{\mathbf{x}^\top\mathbf{G}_\star\mathbf{x}}
\cdot\frac{\mathbf{x}^\top\mathbf{G}_\star\mathbf{x}}{\mathbf{x}^\top\wh{\mathbf{G}}\mathbf{x}}
\in\Bigl[\frac{1}{1+\gamma},\frac{\kappa_k}{1-\gamma}\Bigr]
\end{equation*}
by Lemma~\ref{lem:woodbury} and the sandwich. Inversion is order-reversing on SPD
matrices, so $\wh{\mathbf{G}}^{-1}\preccurlyeq(1-\gamma)^{-1}\mathbf{G}_\star^{-1}$.
Taking traces gives the trace bound.

For the bound on $\gamma$, note that $\lambda_{\min}(\mathbf{G}_\star)=\lambda$
(as $k<p$), so
$\gamma\le\norm{\wh{\mathbf{G}}-\mathbf{G}_\star}_2/\lambda$. Let
$\boldsymbol{\Pi}:=\mathbf{V}\mathbf{V}^\top$ and
$\boldsymbol{\Pi}_\star:=\mathbf{U}_k\mathbf{U}_k^\top$. Then
$\mathbf{G}_\star=\boldsymbol{\Pi}_\star\mathbf{C}\boldsymbol{\Pi}_\star+\lambda\mathbf{I}$
and
$\mathbf{V}\wt{\mathbf{Y}}\mathbf{V}^\top
=\boldsymbol{\Pi}\mathbf{C}\boldsymbol{\Pi}+\mathbf{V}\mathbf{Z}\mathbf{V}^\top$, so
\begin{equation*}
\wh{\mathbf{G}}-\mathbf{G}_\star
=(\boldsymbol{\Pi}-\boldsymbol{\Pi}_\star)\mathbf{C}\boldsymbol{\Pi}
+\boldsymbol{\Pi}_\star\mathbf{C}(\boldsymbol{\Pi}-\boldsymbol{\Pi}_\star)
+\mathbf{V}\mathbf{Z}\mathbf{V}^\top,
\end{equation*}
and hence
$\norm{\wh{\mathbf{G}}-\mathbf{G}_\star}_2
\le2\lambda_1\norm{\boldsymbol{\Pi}-\boldsymbol{\Pi}_\star}_2+\norm{\mathbf{Z}}_2$.
For orthogonal projectors of equal rank,
$\norm{\boldsymbol{\Pi}-\boldsymbol{\Pi}_\star}_2
=\norm{(\mathbf{I}-\boldsymbol{\Pi})\boldsymbol{\Pi}_\star}_2=\zeta$,
which is the sine of the largest principal angle
\citep[Thm.~I.5.5]{stewart1990matrix}. The tail bound is
Lemma~\ref{thm:symmetric_concentration} with $t=2\sqrt{\log(2/\rho)}$.
\end{proof}

\subsection{Proof of Theorem~\ref{thm:total}}
\begin{proof}
Condition on $\mathbf{P}$. Let $f_j:=\mathbf{C}_\lambda^{1/2}e_j$, so that
$\norm{f_j}_2^2=2(\mathcal{L}(\theta_j)-\mathcal{L}(\theta_*))$, and let
$\mathbf{S}:=\mathbf{C}_\lambda^{1/2}\mathbf{P}^{-1}\mathbf{C}_\lambda^{1/2}$.
The matrix $\mathbf{S}$ is similar to $\mathbf{P}^{-1}\mathbf{C}_\lambda$, so its spectrum
lies in $[a,b]$. Since $\eta b\le1$, this gives
$\mathbf{0}\preccurlyeq\mathbf{I}-\eta\mathbf{S}\preccurlyeq(1-\eta a)\mathbf{I}$.
The error evolves as
$f_{j+1}=(\mathbf{I}-\eta\mathbf{S})f_j-\eta\mathbf{C}_\lambda^{1/2}\mathbf{P}^{-1}\mathcal{E}_j$.
The $\mathcal{E}_j$ are independent, zero-mean and independent of $\mathbf{P}$.
With $\mathbf{N}:=\mathbf{C}_\lambda^{1/2}\mathbf{P}^{-2}\mathbf{C}_\lambda^{1/2}\succcurlyeq0$,
this gives
\begin{equation*}
\mathbb{E}\norm{f_T}_2^2=\norm{(\mathbf{I}-\eta\mathbf{S})^Tf_0}_2^2
+\eta^2\sigma_g^2\sum_{j=0}^{T-1}\tr\bigl((\mathbf{I}-\eta\mathbf{S})^{2j}\mathbf{N}\bigr).
\end{equation*}
The first term is at most $(1-\eta a)^{2T}\norm{f_0}_2^2$. For the second term,
note that $\eta\mathbf{S}\preccurlyeq\mathbf{I}$, so
$$
\sum_{j\ge0}(\mathbf{I}-\eta\mathbf{S})^{2j}
=\bigl(\eta\mathbf{S}(2\mathbf{I}-\eta\mathbf{S})\bigr)^{-1}
\preccurlyeq(\eta\mathbf{S})^{-1}.
$$
Since $\mathbf{A}\mapsto\tr(\mathbf{A}\mathbf{N})$ is monotone for a positive semidefinite $\mathbf{N}$, we have that
$$
\sum_{j=0}^{T-1}\tr\bigl((\mathbf{I}-\eta\mathbf{S})^{2j}\mathbf{N}\bigr) \leq
\eta^{-1}\tr(\mathbf{S}^{-1}\mathbf{N})
=\eta^{-1}\tr(\mathbf{C}_\lambda^{-1/2}\mathbf{P}^{-1}\mathbf{C}_\lambda^{1/2})
=\eta^{-1}\tr(\mathbf{P}^{-1}).
$$
Dividing by $2$ shows the claim.
\end{proof}

\subsection{Proof of Corollary~\ref{cor:compare}}
\begin{proof}
(i) Take $\eta=1/b$ in Theorem~\ref{thm:total} and write $\kappa:=b/a$.
For DP-GD, $\kappa=\kappa_0$ and
$\eta\tr(\mathbf{I})=p/(\lambda_1+\lambda)=N_{\mathrm{GD}}/\kappa_0$.
For $\wh{\mathbf{G}}$, Lemma~\ref{lem:perturb} gives $a=\frac{1}{1+\gamma}$ and
$b=\frac{\kappa_k}{1-\gamma}$, so $\kappa=\wh\kappa_k$ and
$\eta\tr(\wh{\mathbf{G}}^{-1})\le\tr(\mathbf{G}_\star^{-1})/\kappa_k=N_k/\wh\kappa_k$.
In both cases, using $\sigma_g^2=T\Delta_g^2/\mu_g^2$ and
$(1-1/\kappa)^{2T}\le e^{-2T/\kappa}$,
\begin{equation*}
\mathbb{E}\,\mathcal{L}(\theta_T)-\mathcal{L}(\theta_*)
\le e^{-2T/\kappa}R_0+\frac{T\Delta_g^2N}{2\mu_g^2\kappa}.
\end{equation*}
Let $\beta:=\log\frac{4\mu_g^2R_0}{\Delta_g^2N}\ge0$ and $T=\lceil\kappa\beta/2\rceil$.
Then $e^{-2T/\kappa}R_0\le e^{-\beta}R_0=\frac{\Delta_g^2N}{4\mu_g^2}$. Also
$T\le\kappa\beta/2+1$, so the second term is at most
$\frac{\Delta_g^2N}{4\mu_g^2}(\beta+2/\kappa)$. Since $\kappa\ge1$, the total is
at most $\Phi(N)$.

(ii) By Lemma~\ref{lem:woodbury},
$\tr(\mathbf{G}_\star^{-1})\le p/\lambda=(1+\lambda_p/\lambda)N_{\mathrm{GD}}$.
For $\gamma=\lambda_p=0$, we have $N_{\mathrm{GD}}=p/\lambda$ and
$N_k=\frac{p}{\lambda}-\sum_{i\le k}\bigl(\frac1\lambda-\frac{1}{\lambda_i+\lambda}\bigr)$.
\end{proof}

\section{Extension of the Convergence Analysis to Generalized Linear Models}
\label{app:glm}

The analysis of Section~\ref{sec:convergence} was stated for linear regression
because that is the case in which the contraction constant can be computed
exactly. We now show that the same statements hold for general GLMs under a
two-sided curvature condition, with $\kappa_k$ replaced by
$\kappa_k/c$. Lemma~\ref{lem:woodbury} and Lemma~\ref{lem:perturb} concern the
preconditioner alone and are unchanged. Theorem~\ref{thm:total} is replaced by
Theorem~\ref{thm:glm} and Corollary~\ref{cor:glm_noise}.

By the mean value theorem (Section~\ref{sec:convergence}), with
$t_j:=\theta_j-\theta_*$, the DP-FC iteration takes the form
\begin{equation}\label{eq:glm_fp}
\begin{aligned}
    t_{j+1} &= \mathbf{A}_j t_j + \mathbf{B}\mathcal{E}_j,
    \qquad \mathbf{B}=-\eta\wh{\mathbf{G}}^{-1},
    \\
    \mathbf{A}_j &= \mathbf{I}-\eta\,\wh{\mathbf{G}}^{-1}\bigl(\bar{\mathbf{H}}_j+\lambda\mathbf{I}\bigr), \\
    \bar{\mathbf{H}}_j &= \tfrac{1}{n}\sum_{(x,y)\in D}\ell''(\xi_x,y)\,xx^\top,
\end{aligned}
\end{equation}
where $\xi_x$ lies between $x^\top\theta_j$ and $x^\top\theta_*$. Note that
$\mathbf{A}_j$ now varies across iterations and successive $\mathbf{A}_j$ do not
commute. This is exactly why we bound the norm in the fixed metric
$\norm{\cdot}_{\wh{\mathbf{G}}}$, with the product bound
$\norm{\prod_{i<j}\mathbf{A}_i}_{\wh{\mathbf{G}}}\le q^{j}$.

\begin{defn}\label{def:curvature}
We say the loss satisfies the $(c_-,c_+)$-curvature condition if $\ell''(\xi_x,y)\in[c_-,c_+]$ for all $(x,y)\in D$ and all $\xi_x$
arising in \eqref{eq:glm_fp}, for some $0<c_-\le c_+$. After rescaling $\ell$ by
$1/c_+$ we may and do assume $c_+=1$ and write $c:=c_-/c_+\in(0,1]$, so that
\begin{equation}\label{eq:sandwich}
    c\,\mathbf{C} \;\preccurlyeq\; \bar{\mathbf{H}}_j \;\preccurlyeq\; \mathbf{C}
    \qquad \text{for all } j .
\end{equation}
\end{defn}

Condition \eqref{eq:sandwich} is the two-sided strengthening of the one-sided assumption
$0\preccurlyeq\bar{\mathbf{H}}_j\preccurlyeq\mathbf{C}$. It is satisfied for the logistic loss
$\ell''(t)=\varrho(t)(1-\varrho(t))$, with $\varrho$ the sigmoid, whenever the
logits along the path are bounded, $\abs{\xi_x}\le M$. In that case
$c = 4\varrho(M)(1-\varrho(M))$. Boundedness of the logits follows from
$\norm{x}_2\le C$ together with a bound on $\norm{\theta_j}_2$. Since the DP noise is unbounded, such a bound does not follow from the $\lambda$-regularization alone; it holds, e.g., for a projected variant of the iteration, and otherwise we take~\eqref{eq:sandwich} along the iterate path as an assumption. The linear regression case is $c=1$.

\begin{thm}\label{thm:glm}
Assume \eqref{eq:sandwich} and $\gamma<1$ (Lemma~\ref{lem:perturb}).
Then every eigenvalue $\mu$ of $\wh{\mathbf{G}}^{-1}(\bar{\mathbf{H}}_j+\lambda\mathbf{I})$
is real and satisfies
\begin{equation*}
    \frac{c}{1+\gamma} \;\le\; \mu \;\le\; \frac{\kappa_k}{1-\gamma}
    \qquad\text{for every } j.
\end{equation*}
Consequently, for $0<\eta\le(1-\gamma)/\kappa_k$,
\begin{equation*}
    \sup_j\;\norm{\mathbf{A}_j}_{\wh{\mathbf{G}}}
    \;\le\; q:=1-\frac{\eta c}{1+\gamma},
\end{equation*}
which equals $1-c/\wh\kappa_k$ for $\eta=(1-\gamma)/\kappa_k$.
Setting $c=1$ recovers the range of Lemma~\ref{lem:perturb}.
\end{thm}

\begin{cor}\label{cor:glm_noise}
Under the assumptions of Theorem~\ref{thm:glm}, with $\wh{\mathbf{G}}$ independent
of $(\mathcal{E}_j)$,
\begin{equation*}
\mathbb{E}\norm{t_T}_{\wh{\mathbf{G}}}^2
\le q^{2T}\norm{t_0}_{\wh{\mathbf{G}}}^2
+\frac{(1+\gamma)\,\eta\,\sigma_g^2}{c}\,\tr(\wh{\mathbf{G}}^{-1}),
\end{equation*}
and
$\mathbb{E}\,\mathcal{L}(\theta_T)-\mathcal{L}(\theta_*)
\le\frac{\kappa_k}{2(1-\gamma)}\,\mathbb{E}\norm{t_T}_{\wh{\mathbf{G}}}^2$.
\end{cor}

\begin{proof}[Proof of Theorem~\ref{thm:glm}]
As in the proof of Lemma~\ref{lem:perturb}, $\wh{\mathbf{G}}\succ0$, and the matrix
$\wh{\mathbf{G}}^{-1}(\bar{\mathbf{H}}_j+\lambda\mathbf{I})$ is similar to the
symmetric matrix
$\wh{\mathbf{G}}^{-1/2}(\bar{\mathbf{H}}_j+\lambda\mathbf{I})\wh{\mathbf{G}}^{-1/2}$.
Its eigenvalues are therefore real and lie between the extreme generalized Rayleigh
quotients. By \eqref{eq:sandwich},
$c\,\mathbf{C}+\lambda\mathbf{I}\preccurlyeq\bar{\mathbf{H}}_j+\lambda\mathbf{I}\preccurlyeq\mathbf{C}_\lambda$,
so it suffices to bound the largest quotient for $\mathbf{C}_\lambda$ from above and the
smallest quotient for $c\mathbf{C}+\lambda\mathbf{I}$ from below.

By Lemma~\ref{lem:perturb},
$\max_v \frac{v^\top\mathbf{C}_\lambda v}{v^\top\wh{\mathbf{G}}v}\le\frac{\kappa_k}{1-\gamma}$.

Then, consider the oracle preconditioner. In the eigenbasis
$\mathbf{U}$, both $c\mathbf{C}+\lambda\mathbf{I}$ and $\mathbf{G}_\star$ are
diagonal, so the generalized eigenvalues are
\begin{equation*}
    \frac{c\lambda_i+\lambda}{\lambda_i+\lambda}\;\;(i\le k),
    \qquad
    \frac{c\lambda_i+\lambda}{\lambda}\;\;(i>k).
\end{equation*}
For $i\le k$, $c\lambda_i+\lambda\ge c(\lambda_i+\lambda)$ because $c\le1$, so the
ratio is at least $c$. For $i>k$, the ratio is at least $1\ge c$. Hence
$\min_v \frac{v^\top(c\mathbf{C}+\lambda\mathbf{I})v}{v^\top\mathbf{G}_\star v}\ge c$.
Applying $v^\top\wh{\mathbf{G}}v\le(1+\gamma)v^\top\mathbf{G}_\star v$ from
Lemma~\ref{lem:perturb} gives
$\min_v \frac{v^\top(c\mathbf{C}+\lambda\mathbf{I})v}{v^\top\wh{\mathbf{G}}v}\ge\frac{c}{1+\gamma}$.

The matrix
$\wt{\mathbf{A}}_j=\wh{\mathbf{G}}^{1/2}\mathbf{A}_j\wh{\mathbf{G}}^{-1/2}
=\mathbf{I}-\eta\,\wh{\mathbf{G}}^{-1/2}(\bar{\mathbf{H}}_j+\lambda\mathbf{I})\wh{\mathbf{G}}^{-1/2}$
is symmetric, so $\norm{\mathbf{A}_j}_{\wh{\mathbf{G}}}=\max_i\abs{1-\eta\mu_i}$.
For $\eta\le(1-\gamma)/\kappa_k$, each $1-\eta\mu_i$ is in
$\bigl[0,\,1-\frac{\eta c}{1+\gamma}\bigr]$.
\end{proof}

\begin{proof}[Proof of Corollary~\ref{cor:glm_noise}]
Condition on $\wh{\mathbf{G}}$. Since $\bar{\mathbf{H}}_j$ depends only on
$\theta_j$, the noise $\mathcal{E}_j$ is zero-mean and independent of
$(t_j,\mathbf{A}_j)$. Also
$\norm{\eta\wh{\mathbf{G}}^{-1}\mathcal{E}_j}_{\wh{\mathbf{G}}}^2
=\eta^2\mathcal{E}_j^\top\wh{\mathbf{G}}^{-1}\mathcal{E}_j$. Therefore, by
Theorem~\ref{thm:glm},
\begin{equation*}
\mathbb{E}\norm{t_{j+1}}_{\wh{\mathbf{G}}}^2
\le q^2\,\mathbb{E}\norm{t_j}_{\wh{\mathbf{G}}}^2+\eta^2\sigma_g^2\tr(\wh{\mathbf{G}}^{-1}).
\end{equation*}
Then, using
$1-q^2\ge1-q=\eta c/(1+\gamma)$ gives the first claim. For the second
claim, $\ell''\le1$ gives $\nabla^2\mathcal{L}\preccurlyeq\mathbf{C}_\lambda$, and
$\nabla\mathcal{L}(\theta_*)=0$. Thus
$\mathcal{L}(\theta)-\mathcal{L}(\theta_*)\le\frac12\norm{t}_{\mathbf{C}_\lambda}^2
\le\frac{\kappa_k}{2(1-\gamma)}\norm{t}_{\wh{\mathbf{G}}}^2$
by Lemma~\ref{lem:perturb}.
\end{proof}

\section{Further Details on Experiments} \label{sec:experiment_details}

\subsection{Additional Information on Datasets and Baseline Methods}

\textbf{Datasets.} In addition to well-known benchmark image datasets, we consider low-dimensional sensor dataset UCI-HAR and WISDM. 

\textbf{MNIST}~\citep{lecun1998gradient}. Handwritten digit images of size $28{\times}28$ ($p=784$), used for the rank-estimation illustration in Section~\ref{subsec:exp_dwork}.
\textbf{SVHN (Street View House Numbers)}~\citep{netzer2011svhn}. A real-world image dataset consisting of $32{\times}32$ RGB images of house numbers collected from Google Street View.

\textbf{FashionMNIST, CIFAR-10, CIFAR-100}. For the downstream machine learning experiments in Section~\ref{subsec:dpfc}, we additionally use FashionMNIST~\citep{xiao2017fashion}, with $512$-dimensional feature vectors extracted from a pretrained ResNet-$34$~\citep{he2016deep}, and CIFAR-10/100~\citep{krizhevsky2009learning}, with $768$-dimensional feature vectors extracted from a pretrained ViT-B/16~\citep{dosovitskiy2021image}.

\textbf{UCI Human Activity Recognition (HAR)}~\citep{anguita2013har}. A multivariate time-series dataset for human activity recognition. The raw signals are sampled into windows of $2.56$ second and from each window, a set of $p = 561$ time- and frequency-domain features (means, standard deviations, signal energy, entropy, autoregressive coefficients, spectral properties) are extracted, giving $10  299$ instances.

\textbf{WISDM (Wireless Sensor Data Mining)}~\citep{kwapisz2011activity}. A benchmark for human activity recognition. 
The raw data consists of time series of $(x,y,z)$ accelerometer measurements. We segment the raw signals into windows of 128 samples ($\approx$ 6.4 seconds) with 50\% overlap. 
For each window, we extract a set of time- and frequency-domain statistics including means, variances, medians, spectral energies and correlation coefficients across axes. 
This results in a feature representation of dimension $200$ per window. 
Across all participants, we obtain $17158$ windows, giving a data matrix of shape $17158 \times 200$.

\para{Baseline Methods.} Among prior methods with $O(pk)$ communication cost, we focus our quantitative comparison on the federated DP-PCA method of~\citet{grammenos2020federated}. The remaining $O(pk)$-communication baselines~\citep{guo2021fedpower,li2024federated} are dominated by our approach by construction: they report only {local} privacy guarantees and do not benefit from privacy amplification by summation. Furthermore,~\citet{guo2021fedpower} uses the centralized DP-PCA method of~\citet{dwork2014analyze} as its core building block, and we show below that we improve upon that method directly even when the latter is given the benefit of summation amplification. Similarly,~\citet{li2024federated} uses the DP power iteration as its building block but lacks our second phase for accurately approximating the projected covariance matrix.
We therefore benchmark against two strong baselines: the centralized $O(p^2)$-communication method of~\citet{dwork2014analyze} and the federated $O(pk)$-communication method of~\citet{grammenos2020federated}.

\subsection{Comparison to Locally DP Baseline (Method of~\citet{grammenos2020federated})}
\label{subsec:exp_grammenos}

We next consider the federated setting and compare against the distributed DP-PCA method of~\citet{grammenos2020federated}, which is built on top of the centralized noisy covariance method of~\citet{dwork2014analyze}. In this method, each client locally applies the DP-PCA mechanism to its user-wise covariance matrix, producing a pair $(\mathbf{U}_i, \mathbf{\Sigma}_i)$ with $\mathbf{U}_i \in \mathbb{R}^{p \times r}$ orthogonal and $\mathbf{\Sigma}_i \in \mathbb{R}^{r \times r}$ diagonal. The server then merges these pairs through the procedure described in Algorithm~2 of~\citet{grammenos2020federated}; with $N$ clients, this requires $\approx \log_2 N$ rounds and $\approx N$ pairwise merges.

Since each client carries out the DP randomization before computing its truncated SVD, the method does not benefit from summation of the client-wise contributions. In contrast, our method offloads noise addition to the clients in such a way that the aggregated noise corresponds to that of a single centrally calibrated Gaussian mechanism applied at the server. Consequently, the per-client noise variance scales as $\sigma^2 / N$, and the privacy--utility trade-off does not depend on the number of clients.
Figure~\ref{fig:MNIST2} verifies this empirically: the relative error of~\citet{grammenos2020federated} increases as the number of clients grows from $4$ to $16$. 

\begin{figure}[h!]
\centering
\includegraphics[width=0.45\textwidth]{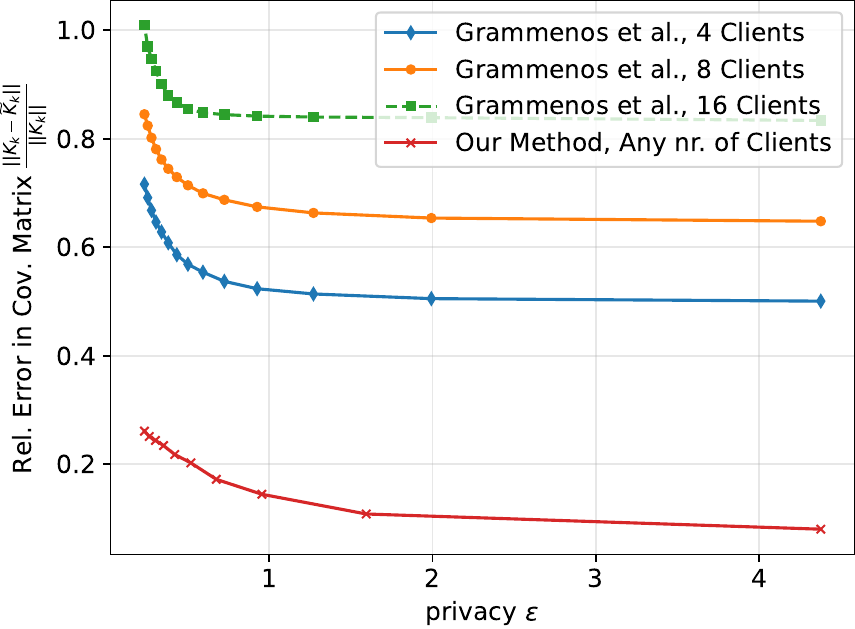}
\caption{Comparison against the federated DP-PCA method of~\citet{grammenos2020federated}. The privacy--utility trade-off of~\citet{grammenos2020federated} {degrades} as the number of clients grows. } 
\label{fig:MNIST2}
\end{figure}

\subsection{Qualitative Illustration: Private Covariance Heatmaps}
\label{subsec:exp_heatmaps}

To complement the quantitative results, Figure~\ref{fig:heat1} visualizes the WISDM covariance matrix recovered by our method and by the Analyze Gauss baseline~\citep{dwork2014analyze} at $(\varepsilon,\delta) = (0.16,10^{-5})$ and rank $k = 20$, alongside the true (non-private) covariance. Our method preserves the block-correlation structure of the original covariance matrix while the baseline produces a heatmap dominated by noise. A similar qualitative comparison for UCI HAR is provided in Figure~\ref{fig:heat2}.

\begin{figure}[h!]
\centering
\includegraphics[width=0.3\columnwidth]{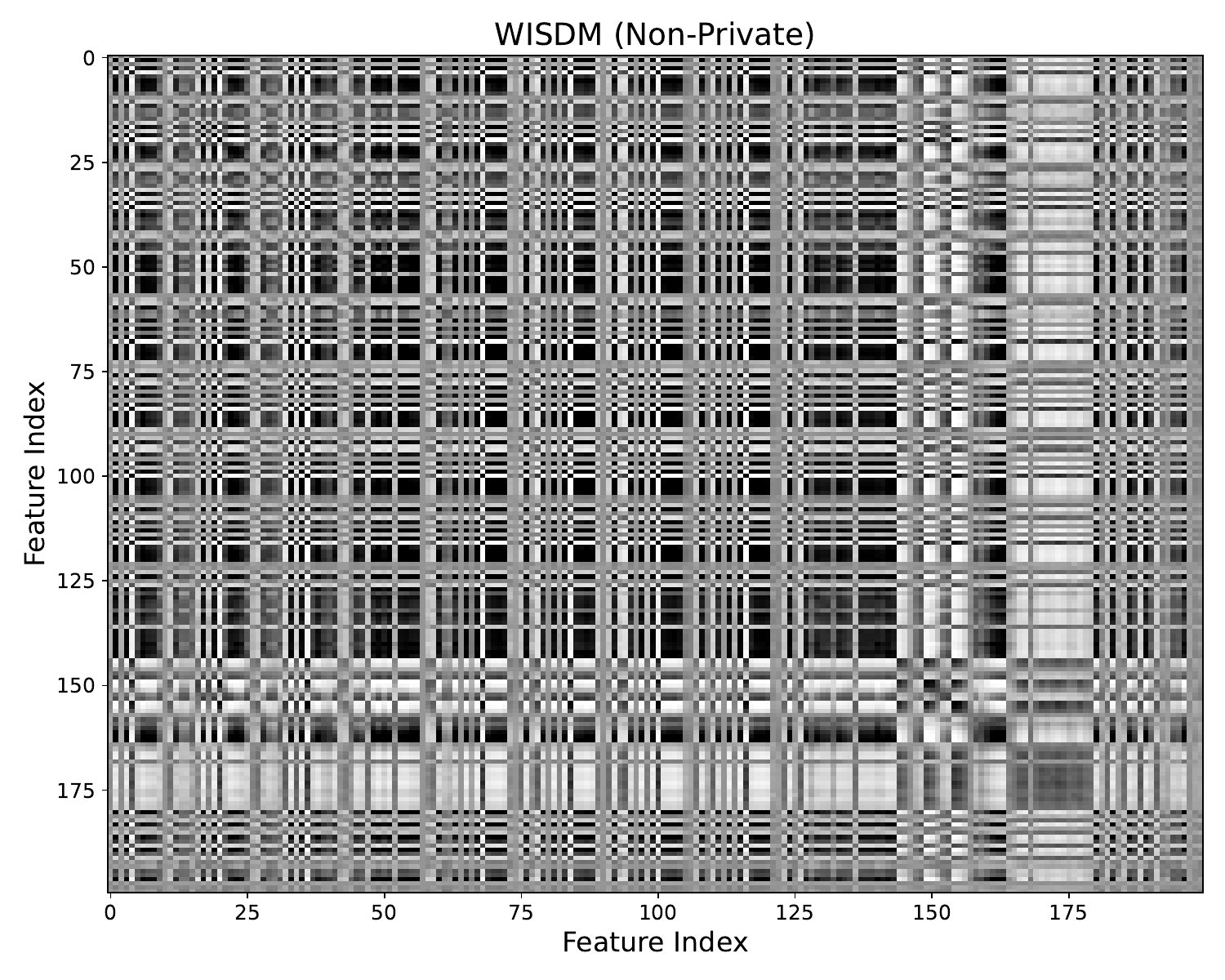}
\includegraphics[width=0.3\columnwidth]{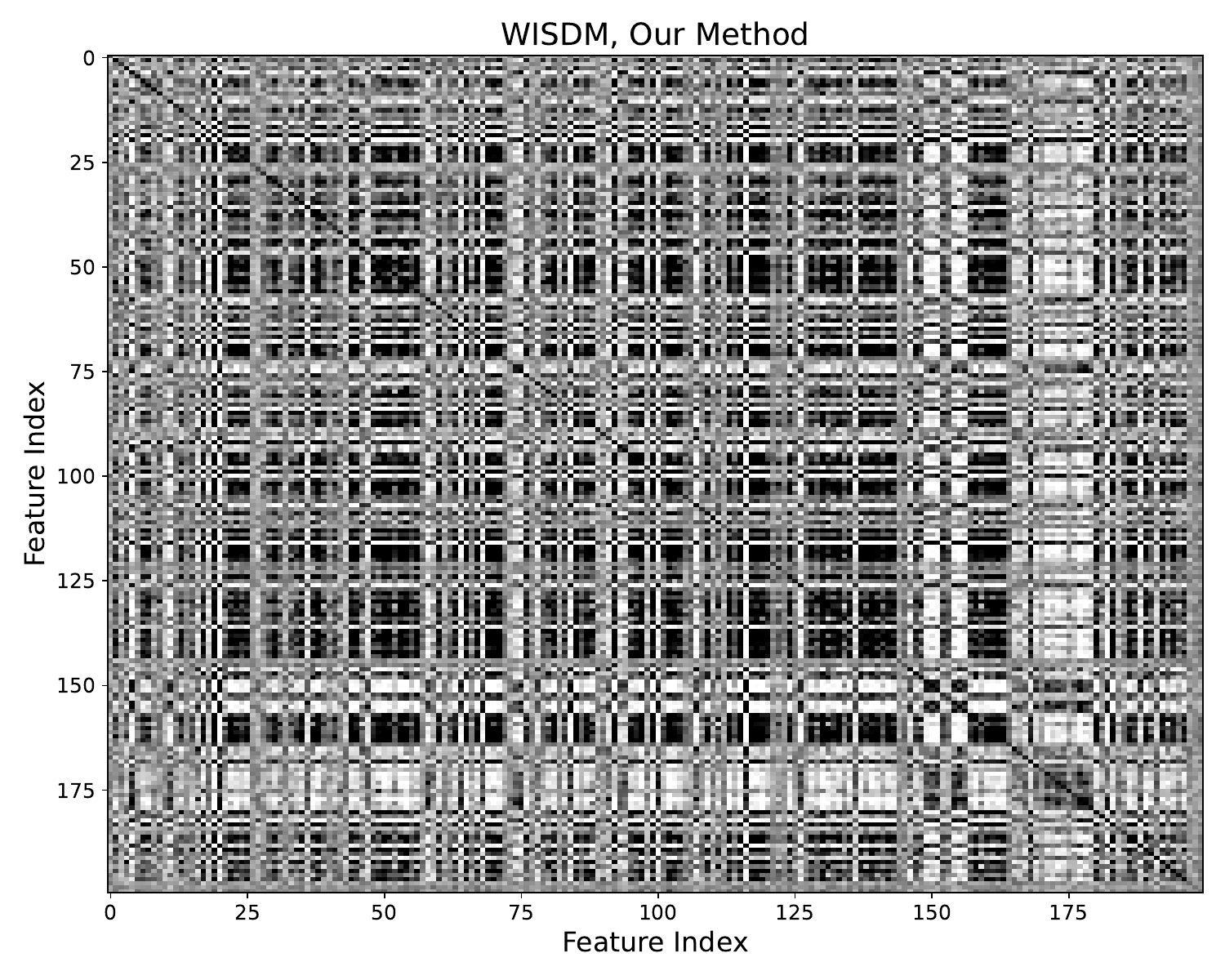}
\includegraphics[width=0.335\columnwidth]{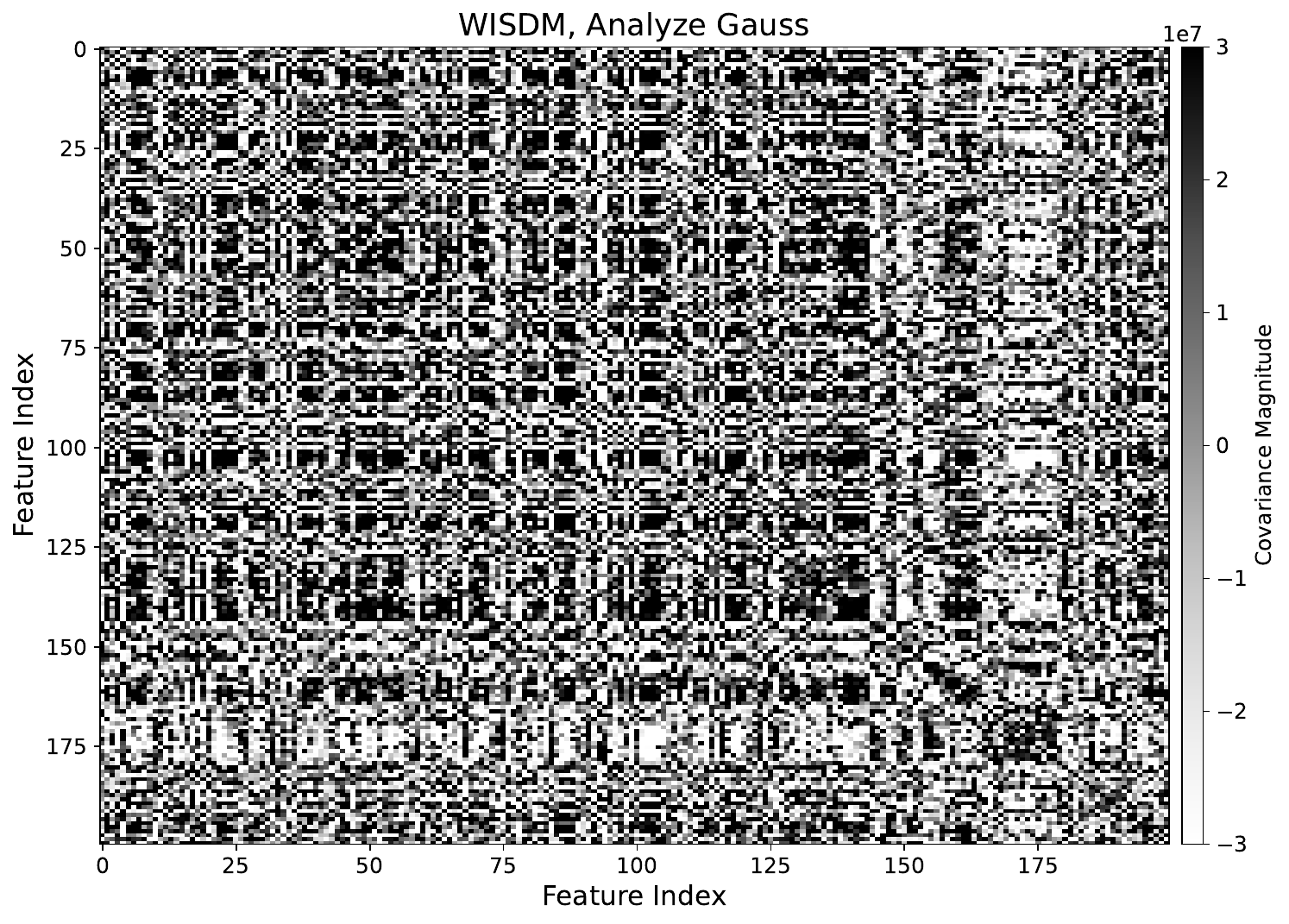}
\caption{Heatmaps of the WISDM covariance matrix when $\varepsilon = 0.16$ and $\delta=10^{-5}$. \textbf{Left:} the true (non-private) covariance. \textbf{Center:} the rank-$20$ approximation produced by our method. \textbf{Right:} the rank-$20$ approximation produced by the Analyze Gauss baseline~\citep{dwork2014analyze}.}
\label{fig:heat1}
\end{figure}

\begin{figure}[h!]
\centering
\includegraphics[width=0.3\columnwidth]{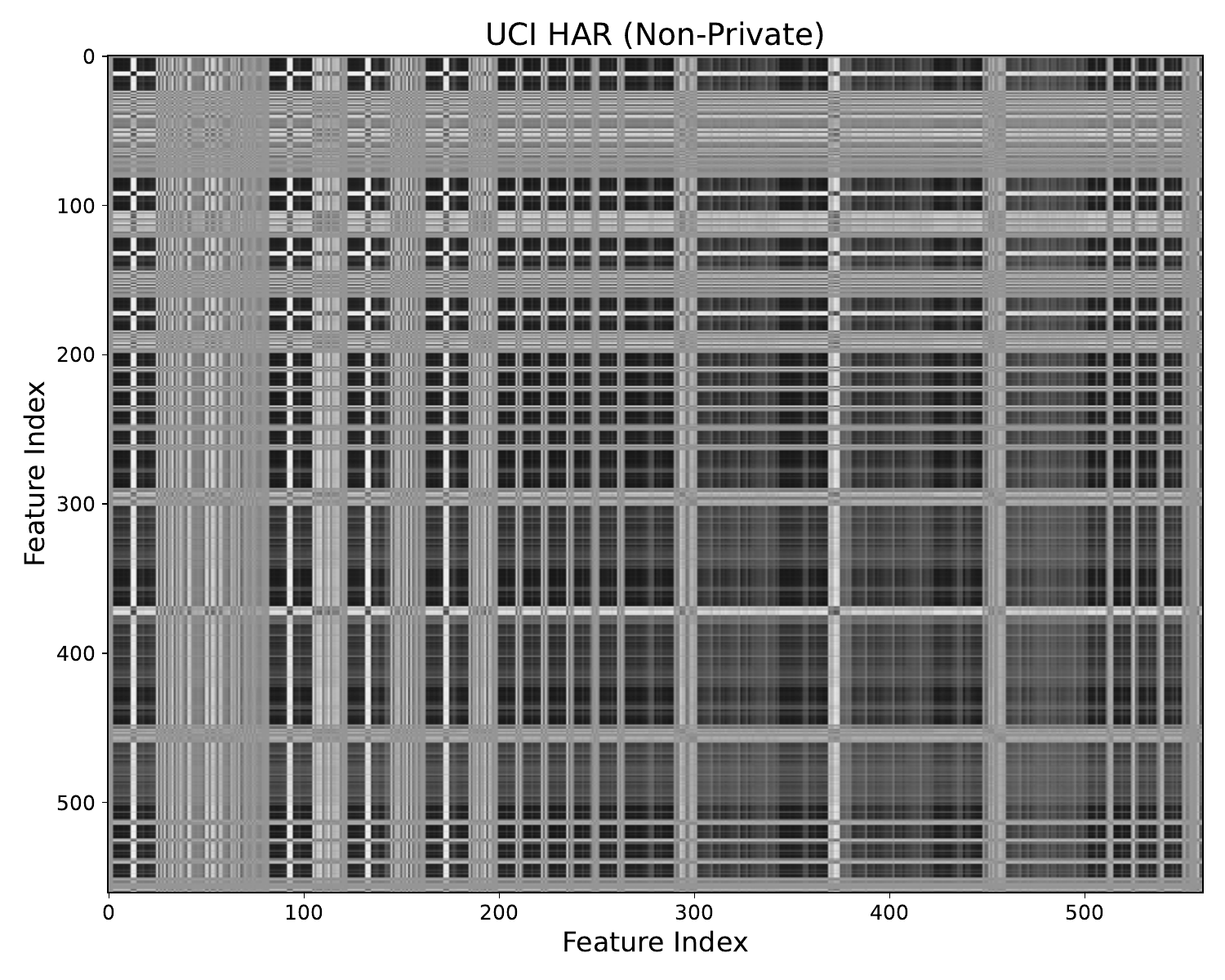}
\includegraphics[width=0.3\columnwidth]{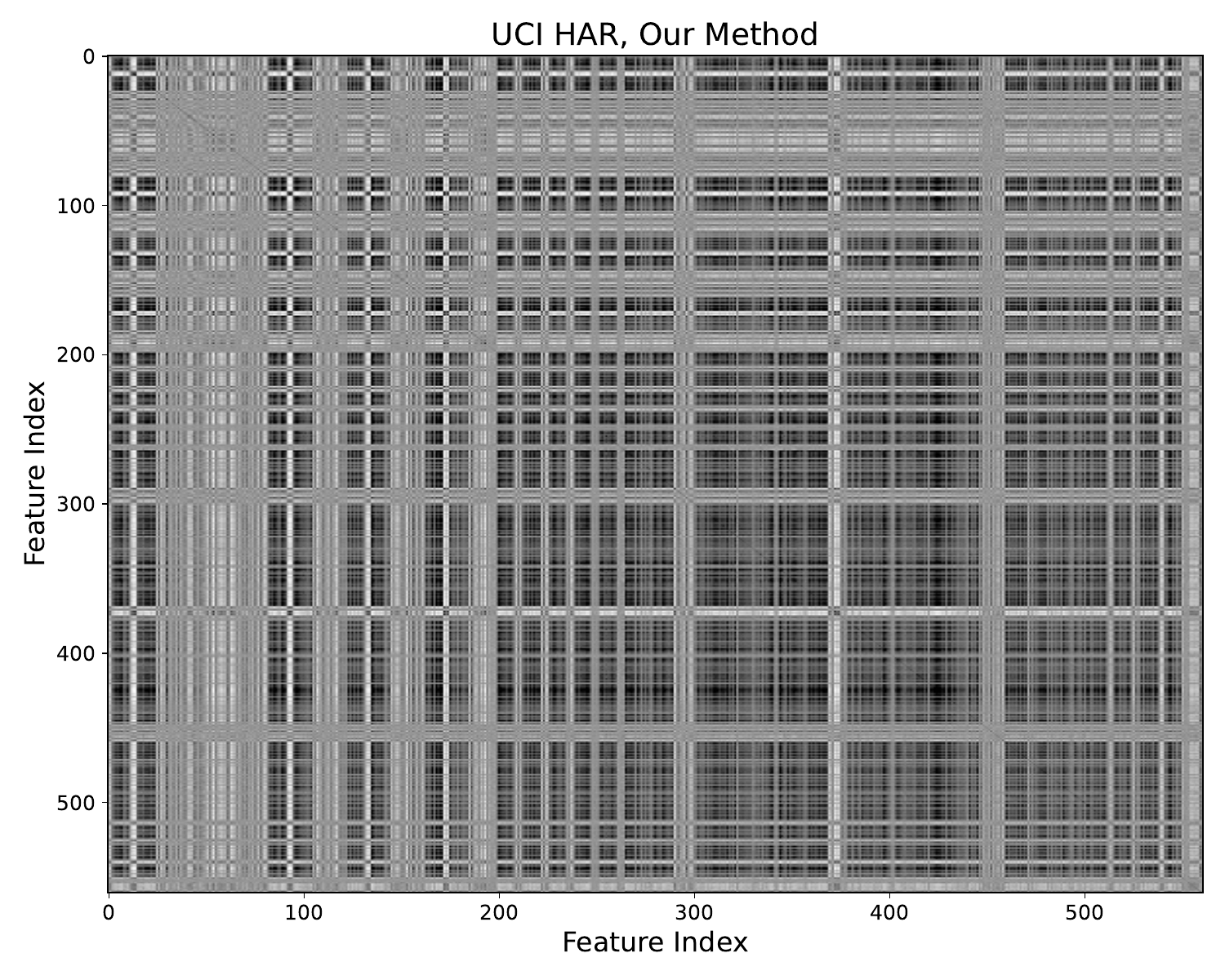}
\includegraphics[width=0.335\columnwidth]{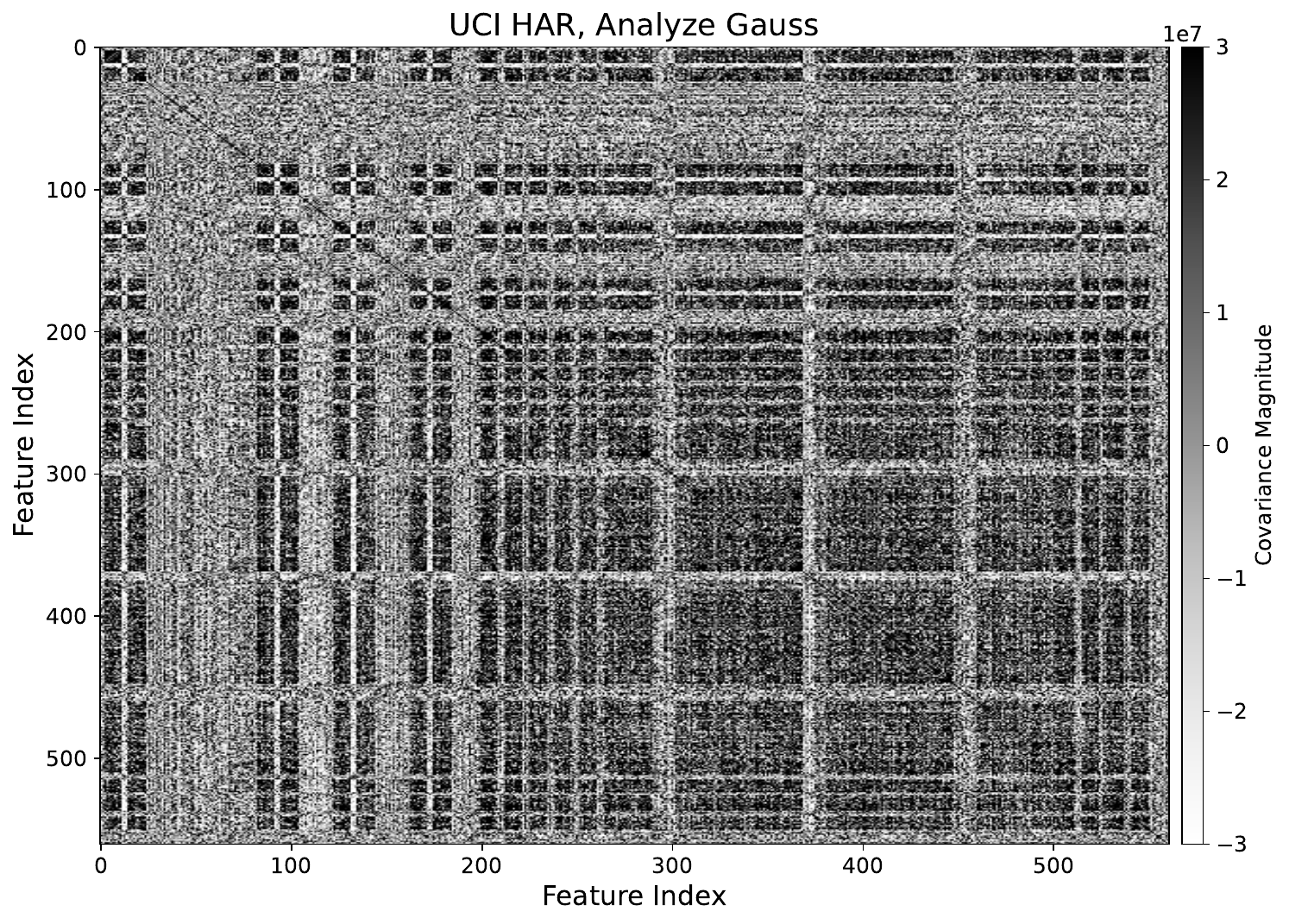}

\caption{The original heatmap of the covariance matrix for UCI HAR dataset features, heatmap obtained using our method with $\varepsilon=0.16$ and heatmap obtained using the baseline Analyze Gauss with $\varepsilon=0.16$ and rank 20.
}   
\label{fig:heat2}
\end{figure}

\clearpage

\subsection{Additional DP-PCA Comparison: SVHN}

Figure~\ref{fig:SVHN} shows the comparison of Analyze Gauss approximation and the DP low rank approximation of the covariance $XX^\top$ when applied to SVHN.


\begin{figure}[h!]
\centering
\includegraphics[width=0.5\columnwidth]{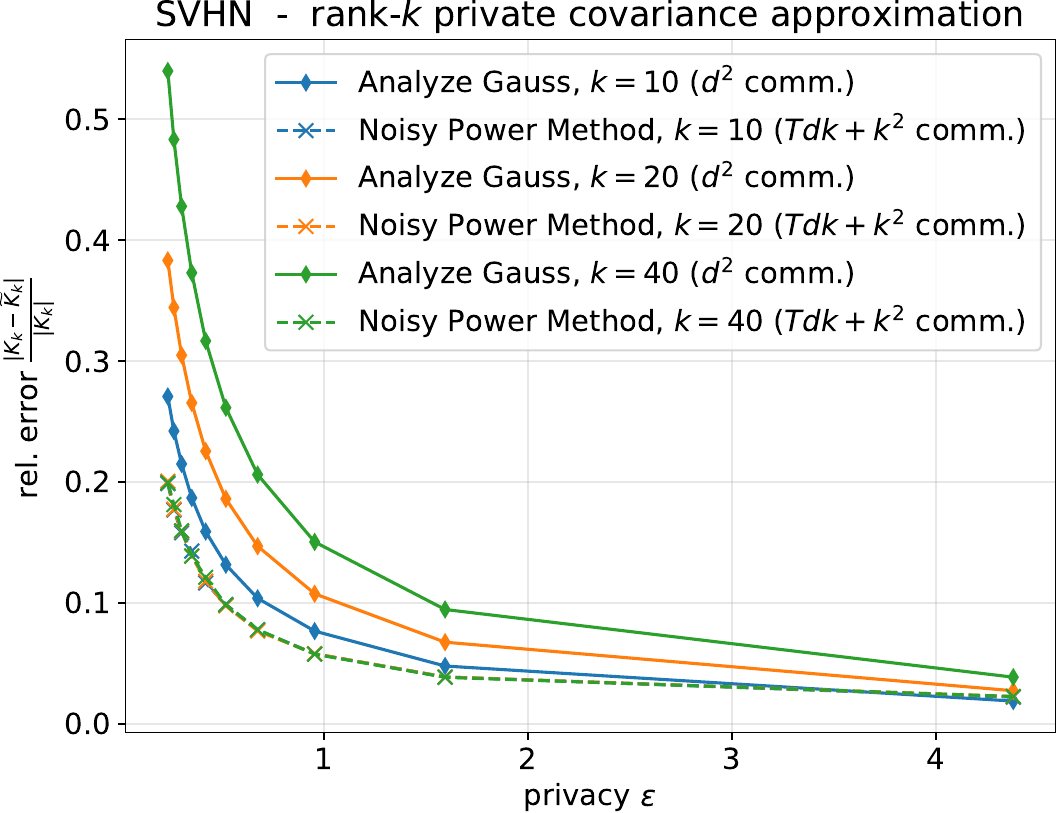} \vspace{3mm}
\caption{Comparison of our method against the central DP-PCA method of~\citet{dwork2014analyze} on SVHN ($p=3072$), with target subspace dimension $k \in \{10, 20, 40\}$. The vertical axis reports the relative Frobenius error of the rank-$k$ covariance approximation, the horizontal axis $\veps$ for $\delta=10^{-5}$. The results are means of 5 independent runs.}
\label{fig:SVHN}
\end{figure}

\clearpage

\section{Complete Experimental Results for Section~\ref{sec:experiments}} \label{sec:app_exp}

\subsection{Number of Communication Rounds vs. Final Accuracy}

 \begin{figure}[h!]
\centering
\includegraphics[width=0.995\columnwidth]{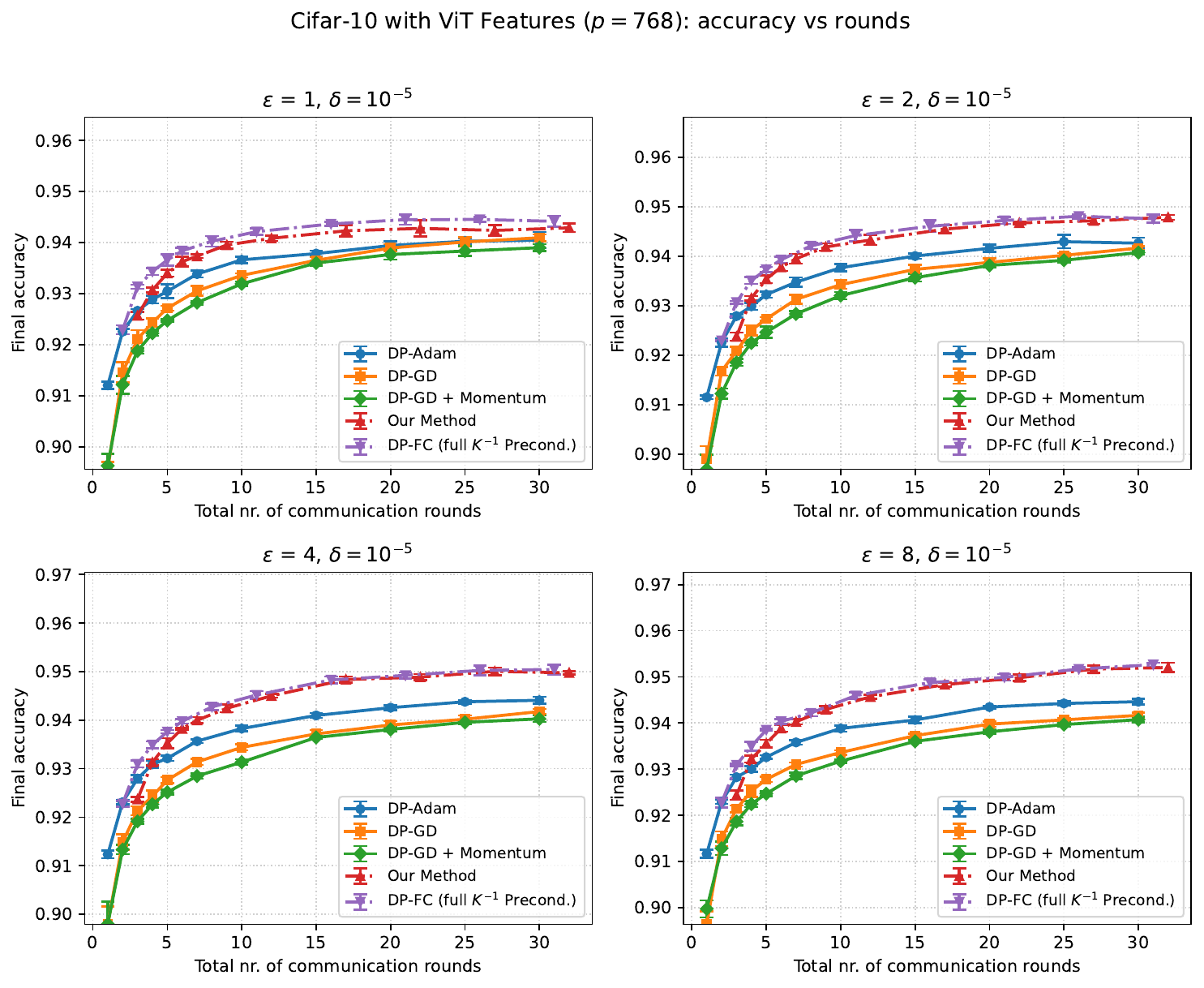}
\caption{CIFAR-10 total communication rounds vs. accuracy. }
\label{fig:cifar10_app_rounds}
\end{figure}

 \begin{figure}[h!]
\centering
\includegraphics[width=0.995\columnwidth]{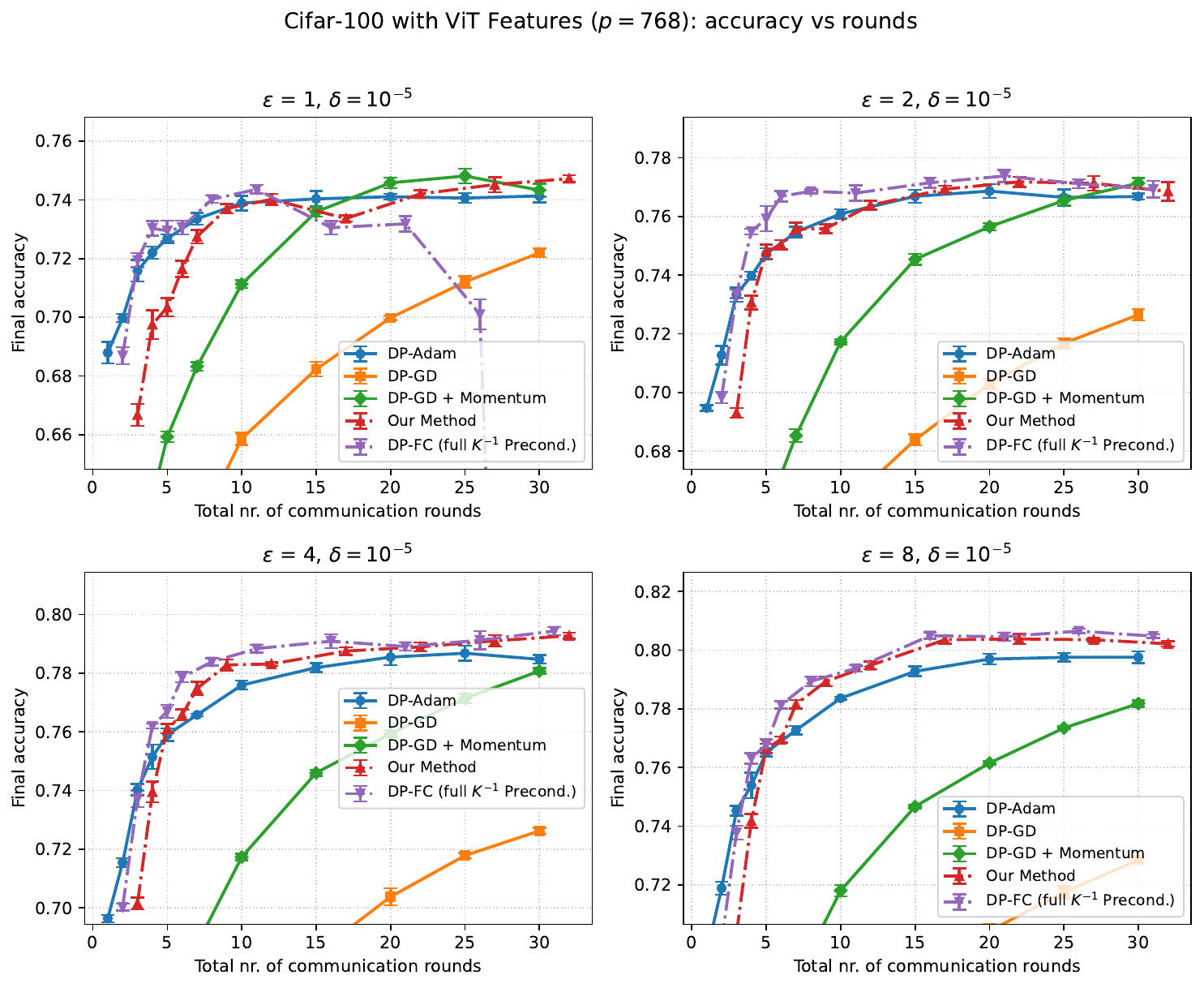}
\caption{CIFAR-100 total communication rounds vs. accuracy. }
\label{fig:cifar100_app_rounds}
\end{figure}

 \begin{figure}[h!]
\centering
\includegraphics[width=0.995\columnwidth]{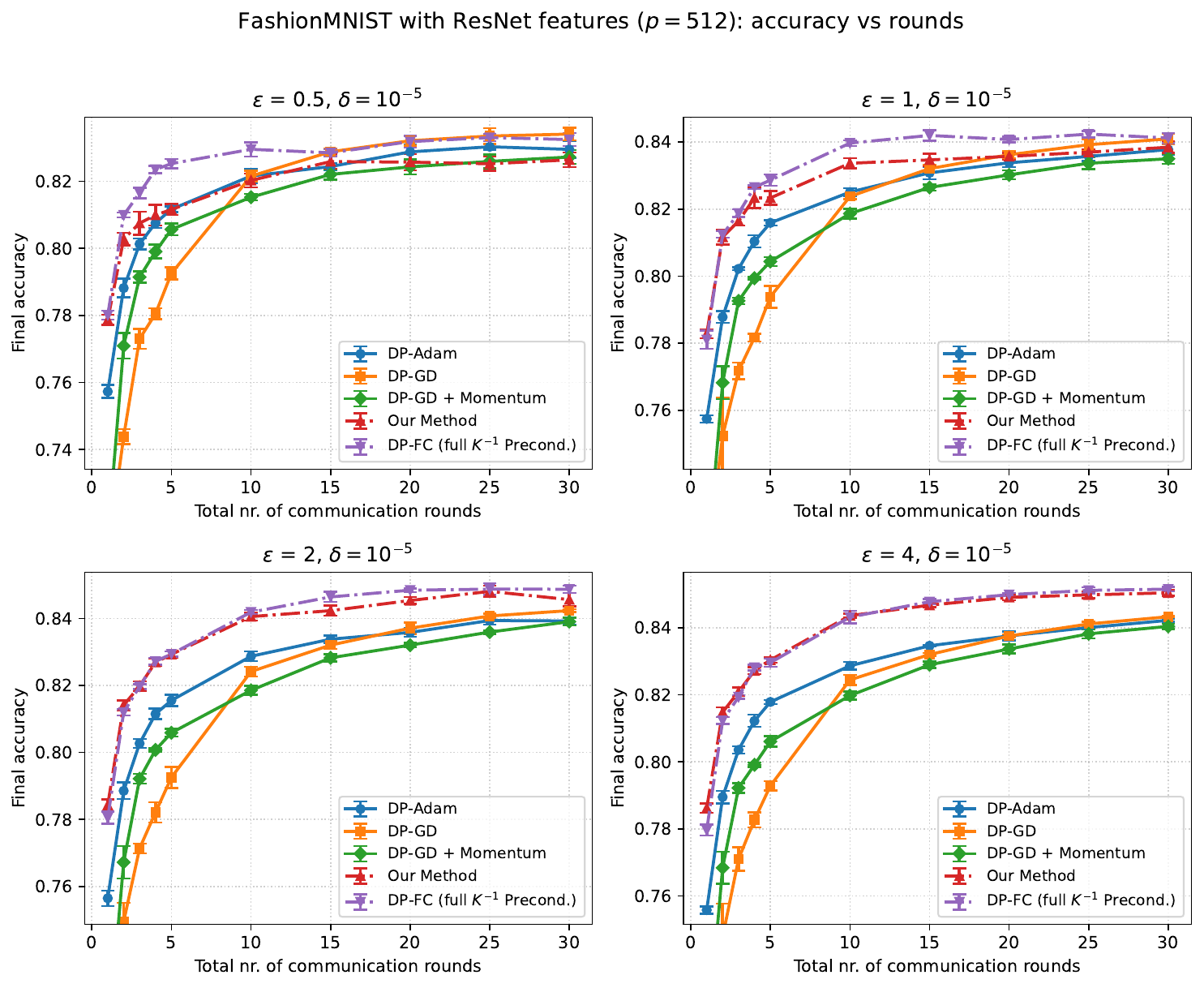}
\caption{FashionMNIST total communication rounds vs. accuracy. }
\label{fig:fmnist_app_rounds}
\end{figure}

\clearpage

\subsection{Total Transmitted Parameters vs. Final Accuracy}

 \begin{figure}[h!]
\centering
\includegraphics[width=0.995\columnwidth]{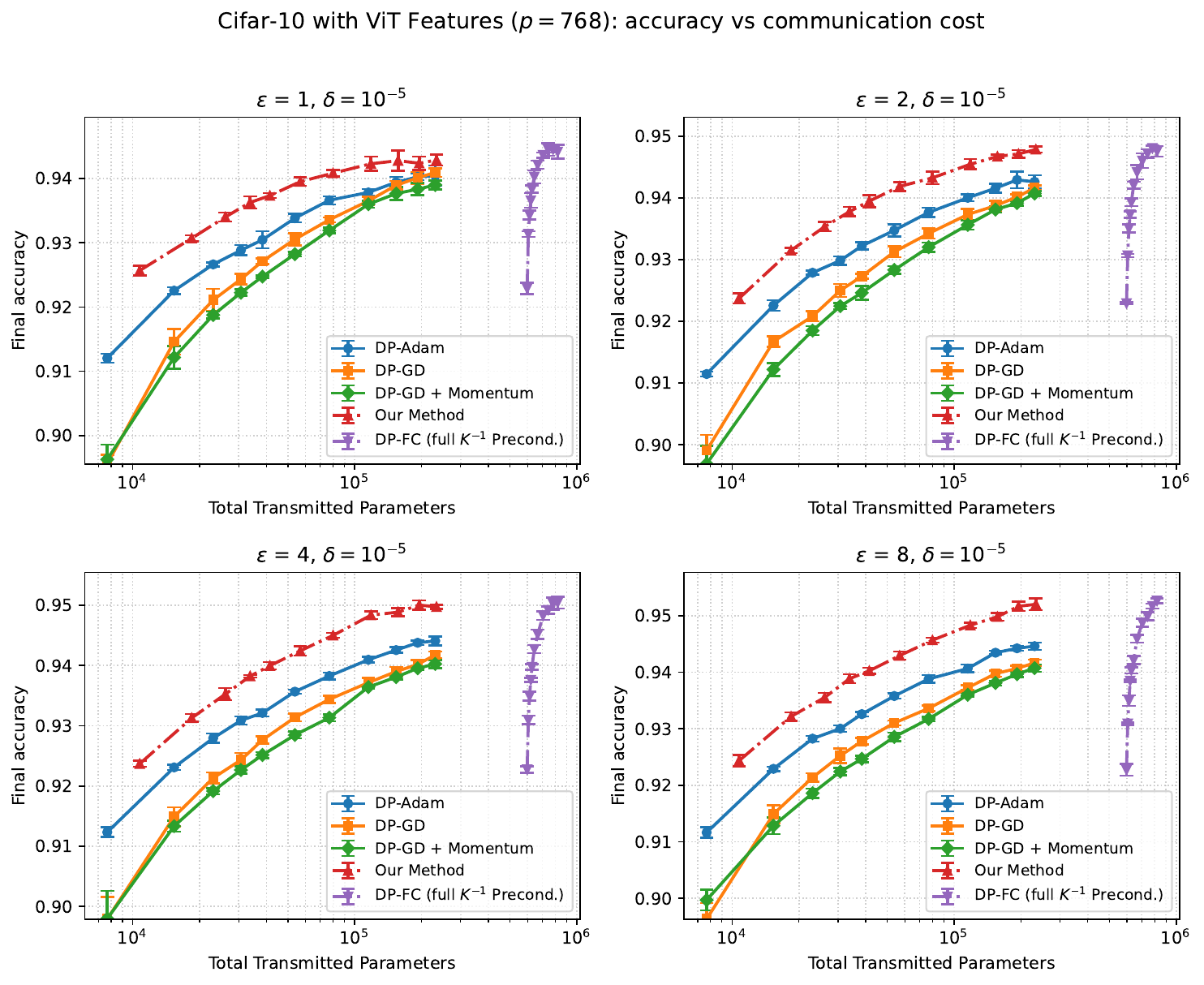}
\caption{CIFAR-10 total transmitted parameters vs. accuracy. }
\label{fig:cifar10_app_comm}
\end{figure}

 \begin{figure}[h!]
\centering
\includegraphics[width=0.995\columnwidth]{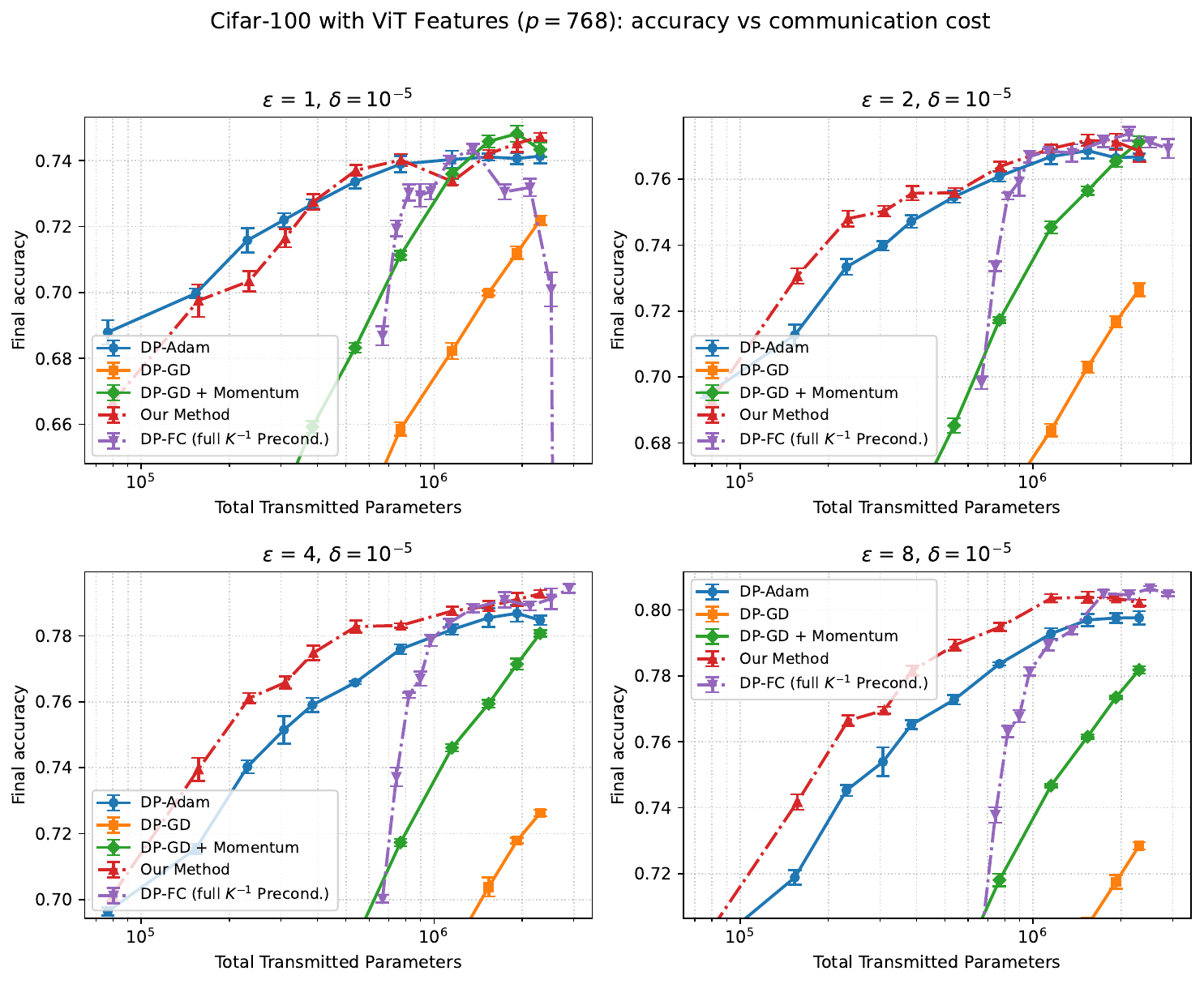}
\caption{CIFAR-100 total transmitted parameters vs. accuracy. }
\label{fig:cifar100_app_comm}
\end{figure}

 \begin{figure}[h!]
\centering
\includegraphics[width=0.995\columnwidth]{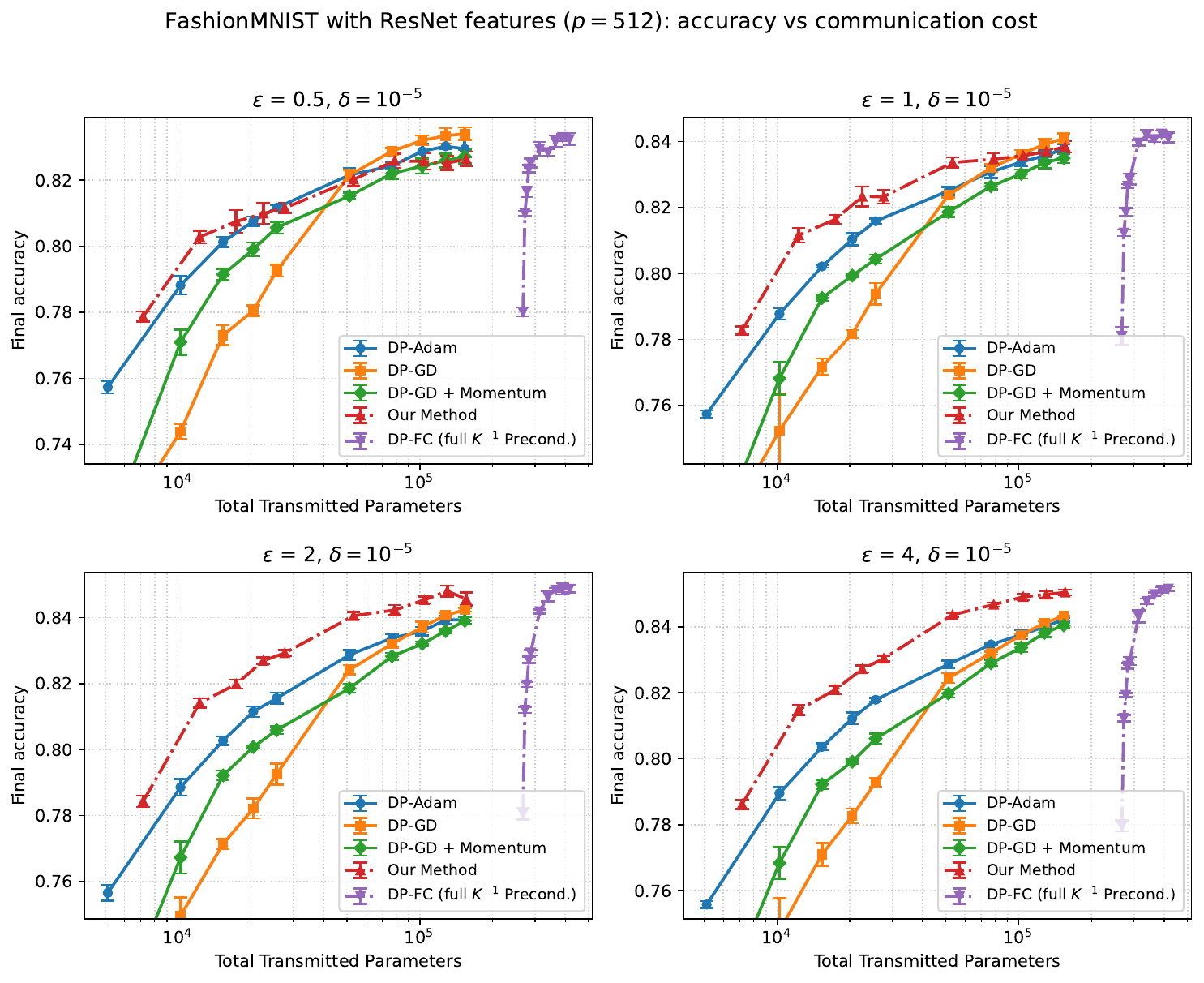}
\caption{FashionMNIST total transmitted parameters vs. accuracy. }
\label{fig:fmnist_app_comm}
\end{figure}

\clearpage

\section{Pseudocodes for Algorithms}

\begin{algorithm}[h!]
\caption{DP Power Method}
\label{alg:DP_power}
\begin{algorithmic}[1]
\REQUIRE Dataset $S=\{\mathbf{x}_i \in \mathbb{R}^{p}\}_{i=1}^n$, noise parameter $\sigma$, number of iterations $T$, block size $k$
\STATE \textbf{Initialize:} Choose $\mathbf{V}_0 \in \mathbb{R}^{p\times k}$ with orthonormal columns (e.g.\ $\mathbf{V}_0=\mathrm{QR}(\mathbf{G})$ for $G_{ij}\sim\mathcal{N}(0,1)$).
\STATE $C$: public bound on $\|\mathbf{x}_i\|_2$ (data clipped to norm $C$).
\STATE $\mathbf{M} \gets \sum_{i=1}^n \mathbf{x}_i \mathbf{x}_i^\top $
\FOR{$t=1,2,\ldots,T$}
  \STATE Sample $\mathbf{Z}_t \in \mathbb{R}^{p\times k}$ with entries i.i.d.\ $\mathcal{N}(0,1)$
  \STATE $\mathbf{V}'_t \gets \mathbf{M} \mathbf{V}_{t-1} + C^2 \sigma \mathbf{Z}_t $
  \STATE $[\mathbf{Q}_t,\mathbf{R}_t] \gets \mathrm{QR}(\mathbf{V}'_t)$ \quad \text{reduced QR}; $\mathbf{Q}_t^{\top}\mathbf{Q}_t = \mathbf{I}_k$
  \STATE $\mathbf{V}_t \gets \mathbf{Q}_t$
\ENDFOR
\RETURN $\mathbf{V}_T$
\end{algorithmic}
\end{algorithm}

\begin{algorithm}[h!]
\caption{Federated DP Power Method with Server-Side QR-Orthogonalization}
\label{alg:fed-dp-power}
\begin{algorithmic}[1]
\REQUIRE $N$ clients with local datasets $S=\{\mathbf{X}_i \in \mathbb{R}^{p\times n_i}\}_{i=1}^{N}$, privacy $(\varepsilon,\delta)$, iterations $T$, block size $k$
\STATE \textbf{Server initialization:} Choose $\mathbf{V}_0 \in \mathbb{R}^{p\times k}$ with orthonormal columns (e.g.\ $\mathbf{V}_0=\mathrm{QR}(\mathbf{G})$, $G_{ij}\!\sim\!\mathcal{N}(0,1)$).
\STATE Let $n=\sum_{i=1}^N n_i$. Let $C$ be a public bound on the 2-norm of the data samples (enforced by clipping).
\FOR{$t=1,2,\ldots,T$}
  \STATE \textbf{Server} broadcasts $\mathbf{V}_{t-1}$ to all clients
  \FOR{\textbf{each client} $m=1,\ldots,N$ \textbf{in parallel}}
    \STATE Compute local moment: \quad $$\mathbf{M}^{(m)} \gets \sum_{j=1}^{n_m} \mathbf{x}^{(m)}_j (\mathbf{x}^{(m)}_j)^{\top} $$
    \STATE Sample $\mathbf{Z}^{(m)}_t \in \mathbb{R}^{p\times k}$ with entries i.i.d.\ $\mathcal{N}(0,1)$
    \STATE Form local contribution: \quad $$\mathbf{Y}^{(m)}_t \gets \mathbf{M}^{(m)}\mathbf{V}_{t-1} + C^2\frac{\sigma}{\sqrt{N}} \mathbf{Z}^{(m)}_t$$
    \STATE Send $\mathbf{Y}^{(m)}_t$ to server \quad (optionally via secure aggregation)
  \ENDFOR
  \STATE \textbf{Server} aggregates: \quad $\mathbf{V}'_t \gets \sum_{m=1}^N \mathbf{Y}^{(m)}_t$
  \STATE \textbf{Server} QR step: \quad $[\mathbf{Q}_t,\mathbf{R}_t] \gets \mathrm{QR}(\mathbf{V}'_t),$
  (reduced QR; $\mathbf{Q}_t^\top\mathbf{Q}_t=\mathbf{I}_k$)
  \STATE \textbf{Server} sets $\mathbf{V}_t \gets \mathbf{Q}_t$
\ENDFOR
\RETURN $\mathbf{V}_T$
\end{algorithmic}
\end{algorithm}

\begin{algorithm}[h!]
\caption{Federated Private Estimation of the Covariance based on Gaussian mechanism~\citep{dwork2014analyze}}
\label{alg:fed-analyze-gauss}
\begin{algorithmic}[1]
\REQUIRE $N$ clients with local examples $\{\mathbf{x}^{(m)}_i \in \mathbb{R}^p\}_{i=1}^{n_m}$, orthonormal basis $\mathbf{V}_k \in \mathbb{R}^{p\times k}$, privacy $(\varepsilon,\delta)$, max 2-norm (or clipping bound) $C>0$
\STATE \textbf{Server:} set $n \gets \sum_{m=1}^N n_m$ and broadcast $\mathbf{V}_k$ to all clients
\FOR{\textbf{each client} $m=1,\ldots,N$ \textbf{in parallel}}
  \STATE Project:
  $$
     \mathbf{Y}^{(m)}_k \gets \mathbf{V}_k^{\top}
     \Big(\frac{1}{n} \sum_{i=1}^{n_m} \mathbf{x}^{(m)}_i \mathbf{x}^{(m)\top}_i \Big)
     \mathbf{V}_k
     \quad \in \mathbb{R}^{k\times k}
  $$
  \STATE Draw $\mathbf{Z}^{(m)} \in \mathbb{R}^{k\times k}$ symmetric with i.i.d.\ $\mathcal{N}(0,1)$ entries on the upper triangular part:
$$
Z^{(m)}_{ij} \stackrel{\text{i.i.d.}}{\sim} \mathcal{N}(0,1) \text{ for } i \le j, \qquad Z^{(m)}_{ji} = Z^{(m)}_{ij} \text{ for } i < j.
$$
  \STATE \textbf{Compute DP release:} 
  $$
  \widetilde{\mathbf{Y}}^{(m)}_k \gets \mathbf{Y}^{(m)}_k + \frac{C^2\sigma'}{n\sqrt{N}} \mathbf{Z}^{(m)}.
  $$
  \STATE Upload $\widetilde{\mathbf{Y}}^{(m)}_k$ to the server
\ENDFOR
\STATE \textbf{Server aggregation:} 
$
\widetilde{\mathbf{Y}}_k \gets \sum_{m=1}^N \widetilde{\mathbf{Y}}^{(m)}_k
$
\RETURN $(\mathbf{V}_k, \widetilde{\mathbf{Y}}_k)$ \quad \textit{(the rank-$k$ private approximation $\mathbf{V}_k\widetilde{\mathbf{Y}}_k\mathbf{V}_k^\top$ of the covariance is formed locally if needed)}
\end{algorithmic}
\end{algorithm}

\begin{algorithm}[h!]
\caption{Private Rank Selection via Random Projection \& Gaussian mechanism}
\label{alg:fed-init-randproj}
\begin{algorithmic}[1]
\REQUIRE $N$ clients with local vectors $\{\mathbf{x}^{(m)}_i \in \mathbb{R}^p\}_{i=1}^{n_m}$, privacy $(\varepsilon,\delta)$, projection dimension $r$, clipping bound $\bar C>0$, failure level $\rho \in (0,1)$
\STATE Server broadcasts a seed and each client forms the (same) projector:
$$
\mathbf{G} \in \mathbb{R}^{p\times r}, \quad \mathbf{G}_{ij} \sim \mathcal{N}(0,1/r)
$$
\FOR{\textbf{each client} $m=1,\ldots,N$ \textbf{in parallel}}
  \STATE Project and clip:
  $$
  \widetilde{\mathbf{x}}^{(m)}_i \gets\mathrm{clip}_{\bar C}(\mathbf{G}^\top \mathbf{x}^{(m)}_i) \in \mathbb{R}^r
  $$
  \STATE Compute the covariance:
  $$
  \mathbf{Y}^{(m)}_r \gets \sum_{i=1}^{n_m} 
  \widetilde{\mathbf{x}}^{(m)}_i \,\widetilde{\mathbf{x}}^{(m)\top}_i 
  \in \mathbb{R}^{r\times r}
  $$
 \STATE Draw client noise: $\mathbf{Z}^{(m)} \in \mathbb{R}^{r\times r}$ symmetric with i.i.d.\ $\mathcal{N}(0,1)$ entries on and above the diagonal:
$$Z^{(m)}_{ij} \stackrel{\text{i.i.d.}}{\sim} \mathcal{N}(0,1) \text{ for } i \le j, \qquad Z^{(m)}_{ji} = Z^{(m)}_{ij} \text{ for } i < j.
$$
  \STATE Compute DP release:
  $$
  \widetilde{\mathbf{Y}}^{(m)}_r \gets \mathbf{Y}^{(m)}_r + \frac{\sigma}{\sqrt{N}} \,\mathbf{Z}^{(m)},
  $$
  where $\sigma$ is adjusted by Lemma~\ref{lem:dp_initialization}.
  \STATE \textbf{Upload} $\widetilde{\mathbf{Y}}^{(m)}_r$.
\ENDFOR
\STATE \textbf{Aggregation at the server:}
$
\widetilde{\mathbf{Y}}_r \gets \sum_{m=1}^N \widetilde{\mathbf{Y}}^{(m)}_r
$
\STATE \textbf{Select the rank:}
Set
$$
\theta_{\mathrm{noise}} \ \gets\ (2\sqrt{r}+t)\,\sigma,
$$
where
$$
t \gets 2\sqrt{\log(2/\rho)}.
$$
Compute the spectrum:
$$
\lambda_1(\widetilde{\mathbf{Y}}_r)\ge \cdots \ge \lambda_r(\widetilde{\mathbf{Y}}_r).
$$
Define
$$
  \mathcal{I} \ \gets\ \Big\{\, i \in \{1,\ldots,r\}\ :\ \lambda_i(\widetilde{\mathbf{Y}}_r)\ >\ \theta_{\mathrm{noise}} \Big\},
$$
and set
$$
\widehat{k} \gets |\mathcal{I}|.
$$
\RETURN The estimated rank $\widehat{k}$
\end{algorithmic}
\end{algorithm}

\end{document}